\documentclass[11pt]{article}
\usepackage[utf8]{inputenc}
\usepackage{amsmath,amssymb}
\usepackage{natbib}
\usepackage{mathrsfs}
\usepackage{nameref}
\usepackage{algorithm}
\usepackage{algorithmic}
\usepackage{tikz}
\usepackage{bm} 
\usepackage{accents}
\usepackage{color}
\usepackage{graphicx}
\usepackage{outlines}
\usepackage{jmlr2e}
\usepackage{mathtools}
\mathtoolsset{showonlyrefs=true}
\makeatletter
\newcommand{\norm}{\@ifstar\@opnorms\@opnorm}
\newcommand{\@opnorms}[1]{%
  \left|\mkern-1.5mu\left|\mkern-1.5mu\left|
   #1
  \right|\mkern-1.5mu\right|\mkern-1.5mu\right|
}
\newcommand{\@opnorm}[2][]{%
  \mathopen{#1|\mkern-1.5mu#1|\mkern-1.5mu#1|}
  #2
  \mathclose{#1|\mkern-1.5mu#1|\mkern-1.5mu#1|}
}
\makeatother
\newcommand{\SG}{\mathcal{S}}
\newcommand{\BG}{\mathcal{B}}
\newcommand{\sparsityProb}{p}
\renewcommand{\dim}{K}
\newcommand{\coef}{\bm{\alpha}} 
\newcommand{\dict}{\bm{D}}
\newcommand{\signal}{\bm{x}} 
\newcommand{\refdict}{\bm{D}^*}

\newcommand{\samplesize}{n}
\newcommand{\size}{K}

\newcommand{\E}{\mathbb{E}}
\newcommand{\N}{\mathcal{N}}
\newcommand{\inv}[1]{#1^{-1}}
\newcommand{\sgn}{\mathrm{sign}}

\newcommand{\BASIS}[1]{\mathbb{B}(#1)}

\newcommand{\R}{\mathbb{R}}
\newcommand{\I}{\mathbb{I}}
\newcommand{\1}{\bm{1}}

\newcommand{\tr}{\mathrm{tr}}

\begin{document}
\title{Unique Sharp Local Minimum in $\ell_1$-minimization Complete Dictionary Learning}
\author{
\name Yu Wang
\email wang.yu@berkeley.edu\\
\addr Department of Statistics\\
University of California, Berkeley\\
Berkeley, CA 94720-1776, USA
\AND 
\name Siqi Wu *
\email siqi@stat.berkeley.edu\\
\addr Citadel Securities\\
131 South Dearborn\\
Chicago, IL 60603, USA
\AND 
\name Bin Yu
\email binyu@berkeley.edu\\
\addr Department of Statistics and EECS\\
University of California, Berkeley\\
Berkeley, CA 94720-1776, USA}  \footnote{The views expressed in this paper reflect those of Siqi Wu and should not be interpreted to represent the views of Citadel Securities or its personnel.}
\editor{}
\maketitle
\begin{abstract}
We study the problem of globally recovering a dictionary from a set of signals via $\ell_1$-minimization. We assume that the signals are generated as {\it i.i.d.} random linear combinations of the $K$ atoms from a complete reference dictionary $\refdict\in \R^{\dim\times\dim}$, where the linear combination coefficients are from either a Bernoulli type model or exact sparse model. First, we obtain a necessary and sufficient norm condition for the reference dictionary $\refdict$ to be a sharp local minimum of the expected $\ell_1$ objective function. Our result substantially extends that of \cite{Wu2015}  and allows the combination coefficient to be non-negative. Secondly, we obtain an explicit bound on the region within which the objective value of the reference dictionary is minimal. Thirdly, we show that the reference dictionary is the unique sharp local minimum, thus establishing  the first known global property of $\ell_1$-minimization dictionary learning. Motivated by the theoretical results, we introduce a perturbation based test to determine whether a dictionary is a sharp local minimum of the objective function. In addition, we also propose a new dictionary learning algorithm based on Block Coordinate Descent, called DL-BCD, which is guaranteed to have monotonic convergence. Simulation studies show that DL-BCD has competitive performance in terms of recovery rate compared to many state-of-the-art dictionary learning algorithms. 
\end{abstract}
\begin{keywords}
dictionary learning, $\ell_1$-minimization, local and global identifiability, non-convex optimization, sharp local minimum
\end{keywords}
\section{Introduction}
Dictionary learning is a class of unsupervised learning algorithms that learn a data-driven representation from signals such as images, speech, and video. It has been widely used in many applications ranging from image imputation to texture synthesis \citep{Rubinstein2010,Mairal2009a,Peyre2009}. Compared to pre-defined dictionaries, data-driven dictionaries exhibit enhanced performance in blind source separation, image denoising and matrix completion. See, e.g., \cite{Zibulevsky2001,Kreutz-Delgado2003,Lesage2005,Elad2006,Aharon2006,
Mairal2009,Qiu2014} and the references therein. 

Dictionary learning can extract meaningful and interpretable patterns from scientific data \citep{Olshausen1996,Olshausen1997,Brunet2004,Wu2016}. The pioneering work of \cite{Olshausen1996} proposed a dictionary learning formulation that minimizes the sum of squared errors with the $\ell_1$ penalty on the linear coefficients to promote sparsity \citep{Tibshirani1991}. Using patches of natural images as signals, their learned dictionary is spatially localized, oriented, and bandpass, similar to the properties of receptive fields found in the primal visual cortex cells. Besides sparsity penalty, researchers also use non-negative constraints, which can efficiently discover part-based representations, leading to a line of research called non-negative matrix factorization (NMF) \citep{Lee2001}. \cite{Brunet2004} demonstrated that NMF has better performance in clustering context-dependent patterns in complex biological systems compared to self-organizing maps and other methods. 
\cite{Wu2016} combined NMF with a stability model selection criteria (staNMF) to analyze {\it Drosphila} early stage embryonic gene expression images. The learned patterns provide a biologically interpretable representation of gene expression patterns and are used to generate local transcription factor regulatory networks. 


{Despite many successful applications, dictionary learning formulations and algorithms are generally hard to analyze due to their non-convex nature. With different initial inputs, a dictionary learning algorithm typically outputs different dictionaries as a result of this non-convexity. 
For those who use the dictionary as a basis for downstream analyses, the choice of the dictionary may significantly impact the final conclusions. Therefore, a natural question to ask is: which dictionary should be used among different output dictionaries of an algorithm?}


Theoretical properties of dictionary learning have been studied under certain data generating models. In a number of recent works, the signals are generated as linear combinations of the columns of the true reference dictionary \citep{Gribonval2010,Geng2014,Jenatton2014}. Specifically, denote by  $\refdict\in\R^{d \times \dim}$ the reference dictionary and $\signal^{(i)}\in \R^d, i = 1,\ldots, \samplesize$ the signal vectors, we have:
\begin{align}\label{Eq:signal_model}
\signal^{(i)} \approx \refdict \coef^{(i)},
\end{align}
where $\coef^{(i)}\in \R^\dim$ denotes the sparse coefficient vector. If $K = d$, the dictionary is called \emph{complete}. If the matrix has more columns than rows, i.e., $K > d$, the dictionary is \emph{overcomplete}. 
Under the model \eqref{Eq:signal_model}, for any reasonable dictionary learning objective function, 
the reference dictionary $D^*$ ought to be equal or close to a local minimum.
This wellposedness requirement, also known as {\it local identifiability} of dictionary learning, turns out to be nontrivial. For a complete dictionary and noiseless signals, {\cite{Gribonval2010} studies the following $\ell_1$-minimization formulation:

\begin{align}\label{Eq:L1_original}
\mathrm{minimize}_{\dict, \{\bm \beta^{(i)}\}_{i = 1}^{\samplesize} } &  \sum_{i=1}^\samplesize \|\bm \beta^{(i)}\|_1.\\
\mathrm{subject~to~} & \|\dict_j\|_2 \leq 1, j = 1,\ldots, \size, \nonumber\\
& \signal^{(i)} = \dict\bm \beta^{(i)}, ~ i = 1,\ldots, \samplesize.
\end{align}}

\noindent They proved a sufficient condition for local identifiability under the Bernoulli-Gaussian model. A more refined analysis by \cite{Wu2015} gave a sufficient and almost necessary condition. The sufficient local identifiability condition in  \cite{Gribonval2010} was extended  
to the over-complete case \citep{Geng2014} and the noisy case \citep{Jenatton2014}. 

As most of dictionary learning formulations are solved by alternating minimization, local identifiability does not guarantee that the output dictionary is the reference dictionary. There are only limited results on how to choose an appropriate initialization. \cite{Arora2015} showed that their initialization algorithm guarantees that the output dictionary is within a small neighborhood of the reference dictionary when certain $\mu$-incoherence condition is met. In practice, initialization is usually done by using a random matrix or randomly selecting a set of signals as columns \citep{Mairal2014SPAMSV2.5}. These algorithms are typically run for multiple times and the dictionary with the smallest objective value is selected. 

{The difficulty of initialization is a major challenge of establishing the recovery guarantee that under some generative models, the output dictionary of an algorithm is indeed the reference dictionary.
This motivates the study of {\it global identifiability}. There are two versions of global identifiability. In the first version, we say that the reference dictionary $\refdict$ is globally identifiable with respect to an objective function $L(\cdot)$ if $\refdict$ is a global minimum of $L$. The second version, which is somewhat stricter, is that $\refdict$ is globally identifiable if all local minima of $L$ are the same as $\refdict$ up to sign permutation. If the second version of global identifiability holds, all local minima are global minimum and we do not need to worry about how to initialize.} Thus any algorithm capable of converging to a local minimum can recover the reference dictionary. For some matrix decomposition tasks such as low rank PCA \citep{Srebro2003} and matrix completion \citep{Ge2017}, despite the fact that the objective function is non-convex, the stricter version of global identifiability holds under certain conditions.   
For dictionary learning, several papers proposed new algorithms with theoretical recovery guarantees that ensure the output is close or equal to the reference dictionary. For the complete and noiseless case, \cite{Spielman2013} proposed a linear programming based algorithm that provably recovers the reference dictionary when the coefficient vectors are generated from a Bernoulli Gaussian model and contain at most $O(\sqrt{\dim})$ nonzero elements. 
\cite{Sun2017a,Sun2017b} improved the sparsity tolerance to $O(\dim)$ using a Riemannian trust region method. 
For over-complete dictionaries, \cite{Arora2014a} proposed an algorithm which performs an overlapping clustering followed by an averaging algorithm or a K-SVD type algorithm. 
Additionally, there is another line of research that focuses on the analysis of alternating minimization algorithms, including \cite{Agarwal2013, Agarwal2013a, Arora2014, Arora2015, Chatterji2017}.  
\cite{Barak2014} proposed an algorithm based on sum-of-square semi-definite programming hierarchy and proved its desirable theoretical performance with relaxed assumptions on coefficient sparsity under a series of moment assumptions.

Despite numerous studies of global recovery in dictionary learning, there are no global identifiability results for the $\ell_1$-minimization problem. As we illustrate in Section \ref{Sec:3}, the reference dictionary may not be the global minimum even for a simple data generation model. This motivates us to consider a different condition to distinguish the reference dictionary from other local minima. 
In this paper, we obtain a uniqueness characterization of the reference dictionary. We show that the reference dictionary is the unique ``sharp" local minimum (see Definition \ref{Def:sharplocalmin}) of the $\ell_1$ objective function when certain conditions are met -- in other words, there are no other sharp local minima  than the reference dictionary. 

Based on this new characterization and the observation that a sharp local minimum is more resilient to small perturbations, we propose a method to empirically test the sharpness of objective function at the reference dictionary. Furthermore, we also design a new algorithm to solve the $\ell_1$-minimization problem using Block Coordinate Descent (DL-BCD) and the re-weighting scheme inspired by \cite{Candes2008}. Our simulations demonstrate that  
the proposed method compares favorably with other state-of-the-art algorithms in terms of recovery rate.


Our work differs from other recent studies in two main aspects. Firstly, instead of proposing new dictionary learning formulations, we study the global property of the existing $\ell_1$-minimization problem that is often considered difficult in previous studies \citep{Wu2015,Mairal2009}. Secondly, our data generation models are novel and cover several important cases not studied by prior works, e.g., non-negative linear coefficients. Even though there is a line of research that focuses on non-negative dictionary learning in the literature \citep{Aharon2005,Hoyer2002,Arora2014}, the reference dictionary and the corresponding coefficients therein are both non-negative. 
In comparison, we allow the dictionary to have arbitrary values but only constrain the reference coefficients to be non-negative. This non-negative coefficient case is difficult to analyze and does not satisfy the recovery conditions in previous studies, for instance \cite{Barak2014,Sun2017a,Sun2017b}. 


The rest of this paper is organized as follows. Section \ref{Sec:2} introduces notations and basic assumptions. 
Section \ref{Sec:3} presents main theorems and discusses their implications. 
Section \ref{Sec:4} proposes the sharpness test and the block coordinate descent algorithm for dictionary learning (DL-BCD). Simulation results are provided in Section \ref{Sec:5}.


\section{Preliminaries}\label{Sec:2}

For a vector $\bm w\in \R^m$, denote its $j$-th element by $w_j$. For an arbitrary matrix $A\in\R^{m\times n}$, let $A[k,], A_j, A_{k,j}$ denote its $k$-th row, $j$-th column, and the $(k,j)$-th element respectively. Denote by $A[k, -j]\in \R ^{n-1}$ the $k$-th row of $A$ without its $j$-th entry. Let $\I$ denote the identity matrix of size $\dim$ and for $k\in \{1,\ldots, \dim\}$, $\I_k$ denotes $\I$'s $k$-th column, whose $k$-th entry is one and zero elsewhere. For a positive semi-definite square matrix $X\in \R^{\dim\times \dim}$, $X^{1/2}$ denotes its positive semi-definite square root.  We use $\|\cdot\|$ to denote vector norms and $\norm{\cdot}$ to denote matrix (semi-)norms. In particular, $\norm{\cdot}_F$ denotes matrix's Frobenius norm, whereas $\norm{\cdot}_2$ denotes the spectral norm.
For any two real functions $w(t), q(t):\R \rightarrow \R$, we denote $w(t) = \Theta (q(t))$ if there exist constants $c_1,c_2>0$ such that for any $t\in \R$, $c_1 < \frac{w(t)}{q(t)} < c_2$. If $q(t) > 0$ and $\lim_{q(t)\rightarrow 0} \frac{w(t)}{q(t)} = 0$, then we write $w(t) = o(q(t))$.
Define the indicator and the sign functions as
\[
\1(x = 0) = 
\left\{
\begin{array}{ll}
1 & x = 0\\
0 & x \neq 0
\end{array}
\right.,~
\sgn(x) =  
\left\{
\begin{array}{ll}
1 & x > 0\\
0 & x = 0\\
-1 & x < 0
\end{array}
\right..
\]

In dictionary learning, a dictionary is represented by a matrix $\dict\in\R^{d \times \dim}$. We call a column of the dictionary matrix an atom of the dictionary. In this paper, we consider complete dictionaries, that is, the dictionary matrix is square ($K=d$) and invertible. We also assume noiseless signal generation: denote by $\refdict$ the reference dictionary, the signal vector $\signal$ is generated from a linear model without noise: $\signal = \refdict \coef$.

Define $L$ to be our $\ell_1$ objective function for a complete dictionary $\dict$:

\begin{equation}\label{Eq:L}
L(\dict) = \frac{1}{n}\sum_{i=1}^\samplesize \|\dict^{-1}\signal^{(i)}\|_1.
\end{equation}
{When the dictionary is complete and invertible, it can be shown that the $\ell_1$-minimization formulation \eqref{Eq:L1_original} is equivalent to the following optimization problem:}
\begin{align}\label{Eq:main_finitesample}
\textrm{minimize}_{\dict\in \BASIS{\R^\dim}} L(\dict),
\end{align}
where $\BASIS{\R^\dim}$ is the set of all feasible dictionaries: 
\begin{align}\label{Eq:basis}
\BASIS{\R^\dim} \triangleq \left\{\dict \in \R^{\dim\times \dim}\Big| \|\dict_{1}\|_2 = \ldots = \| \dict_{\dim}\|_2 =  1,~\mathrm{rank}(\dict)=\dim\right\}.
\end{align}



{\noindent Here are several commonly used terminologies in dictionary learning.}
\begin{itemize}
\item {\em Sign-permutation ambiguity.} In most dictionary learning formulations, the order of the dictionary atoms as well as their signs do not matter. Let $P\in\R^{\dim \times \dim}$ be a permutation matrix and $\Lambda\in\R^{\dim \times \dim}$ a diagonal matrix with $\pm 1$ diagonal entries. The matrix $\dict' = \dict P \Lambda$ and $\dict$ essentially represent the same dictionary but $\dict' \neq \dict$ element-wise. 

\item {\em Local identifiability.} The reference dictionary $\refdict\in \BASIS{\R^\dim}$ is \emph{locally identifiable} with respect to $L$ if $\refdict$ is a local minimum of $L$.
Local identifiability is a minimal requirement for recovering the reference dictionary. It has been extensively studied under a variety of dictionary learning formulations \citep{Gribonval2010, Geng2014, Jenatton2014, Wu2015,Agarwal2013a,Schnass2014b}. 

\item{\em Global identifiability.} The reference dictionary $\refdict\in \BASIS{\R^\dim}$ is \emph{globally identifiable} with respect to $L$ if $\refdict$ is a global minimum of $L$. Clearly, whether global identifiability holds depends on the objective function and the signal generation model.
If the objective function is $\ell_0$ and the linear coefficients are generated from the Bernoulli Gaussian model, the reference dictionary is globally identifiable (see Theorem 3 in \cite{Spielman2013}). However, if the objective function is $\ell_1$, global identifiability might not hold. In Section \ref{Sec:3}, we give an example where the reference dictionary is only a local minimum but not a global minimum. This implies that global identifiability of $\ell_1$-minimization requires more stringent conditions. Therefore, we consider a variant of global identifiability, which is to show that, under certain conditions, the reference dictionary $\refdict$ is the \emph{unique sharp local minimum} of the dictionary learning objective function. In other words, no dictionary other than $\refdict$ is a sharp local minimum. Other dictionaries can still be local minima but cannot be {\em sharp} at the same time. This property allows us to globally distinguish the reference dictionary from other spurious local minima and can be used as a criterion to select the best dictionaries from a set of algorithm outputs.
Sharp local minimum, which is defined in Definition \ref{Def:sharplocalmin}, is a common concept in the field of optimization \citep{Dhara2011,polyak79}. However, to the best of our knowledge, we are the first to connect dictionary learning theory with sharp local minimum and use it to distinguish the reference dictionary from other spurious local minima.


\end{itemize}

\begin{definition}[Sharp local minimum] Let $L(\dict):\BASIS{\R^\dim} \rightarrow \R$ be a dictionary learning objective function. 
A dictionary $\dict^0\in \BASIS{\R^\dim}$ is a sharp local minimum of $L(\cdot)$ with sharpness $\epsilon$ if there exists $\delta>0$ such that for any $\dict \in \{\dict: \norm{\dict - \dict^0}_F < \delta\}$:
\[
L(\dict) - L(\dict^0) \geq \epsilon \norm{\dict - \dict^0 }_{F} + o(\norm{\dict^0 - \dict}_{F}).
\]
\label{Def:sharplocalmin}
\end{definition}

{\noindent\textbf{Remarks:}
The definition here can be viewed as a matrix analog of the sharp minimum in the one dimensional case. For a function $f:\R\rightarrow \R$, $v^0$ is a sharp local minimum of $f$ is $f(v) - f(v^0) \geq \epsilon |v - v^0| + o(|v - v^0|)$.
Note that the definition of sharp local minimum is different from the definition of strict local minimum, which means there is no other local minimum in its neighborhood. Sharp local minimum is always a strict local minimum but not vice versa. For example, $x = 0$ is a strict local minimum of the $\ell_q$ function $|x|^q$ for any $q \geq 0$ and is a sharp local minimum when $q \leq 1$. However, it is not a sharp local minimum when $q > 1$. }\\






For any $\dict\in \BASIS{\R^{\dim}}$, define $M(\dict) = \dict^T\dict$ as the collinearity matrix of $\dict$. For example, if the dictionary is an orthogonal matrix, $M(\dict) = \I$ is the identity matrix. If all the atoms in the dictionary has collinearity 1, then $M(\dict) = \1\1^T$ is a matrix whose elements are all ones. When the context is clear, we use $M$ instead of $M(\dict)$ for notation ease. Denote by $M^*$ the collinearity matrix for the reference dictionary $\refdict$. Also, define the matrix $B(\coef, M) \in \R^{\dim\times\dim}$ as
\[
(B(\coef, M))_{k, j} \triangleq \E \coef_j\sgn (\coef_k) - M_{j, k}\E |\coef_j|\quad \text{for $k, j = 1,\ldots, \dim$}.
\]
For any random vector $\coef$, define the semi-norm $\norm{\cdot}_{\coef}$ induced by $\coef$ as:
\[
\norm{A}_{\coef} \triangleq \sum_{k=1}^\dim \E |\sum_{j=1}^\dim A_{k, j} \coef_j|\1(\coef_k = 0).
\]

$\norm{\cdot}_{\coef}$ is a semi-norm but not a norm because $\norm{A}_{\coef} = 0$ does not imply $A = 0$. Actually, for any nonzero diagonal matrix $A\neq 0$, $\norm{A}_{\coef} = 0$ because $\sum_{k=1}^\dim \E |A_{k,k}\coef_k|\1(\coef_k = 0) = 0$. Note that it could be unsettling for readers at first why we define $B$ and $\norm{\cdot}_{\coef}$ this way. That is because these quantities appear naturally in the first order optimality condition of $\ell_1$-minimization. 
Hopefully, the motivation of defining these definitions will become clear later.

\subsection{Assumptions and models}\label{Sec:assumptions_models}
In this subsection, we first introduce two assumptions on the generation scheme for the linear coefficient $\coef$. 
This class includes two widely used models in the dictionary learning literature \citep{Gribonval2010,Wu2015}: Bernoulli Gaussian and sparse Gaussian distributions (Figure \ref{Fig:diagram}).
\begin{figure}[ht]
    \centering
    \includegraphics[width=.5\textwidth]{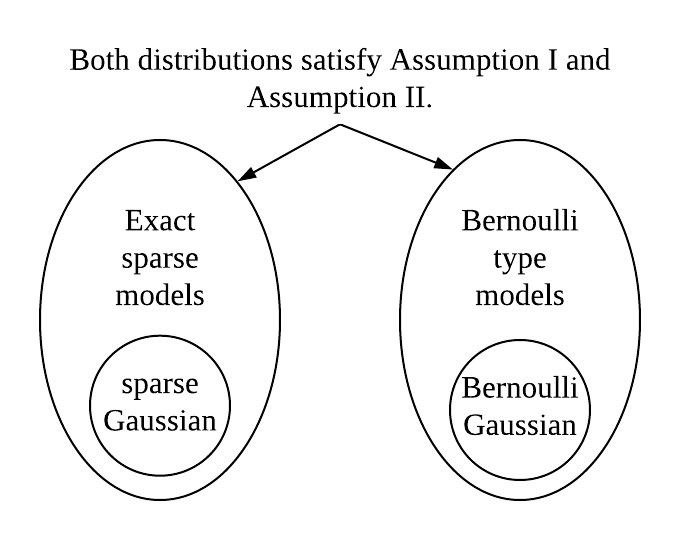}
    \caption{Both exact sparse models and Bernoulli type models satisfy Assumption I and II. Sparse Gaussian distribution is a special case of exact sparse models, while Bernoulli Gaussian distribution is a special case of Bernoulli type models. }
    \label{Fig:diagram}
\end{figure}

\paragraph{Assumption I } $\norm{\cdot}_{\coef}$ is $c_{\coef}$-regular: 
for any matrix $A\in H^\dim$, where $H^{\dim} = \{A\in\R^{\dim\times\dim} \ | \ A_{i,i}=0 \text{ for all } 1\leq i\leq \dim \}$, $\norm{A}_{\coef}$ is bounded below by $A$'s Frobenious norm: $\norm{A}_{\coef} \geq c_{\coef}\norm{A}_F$.
\vspace*{5mm}

{\noindent Assumption I has several implications. First, it ensures that the coefficient vector $\coef$ does not lie in a linear subspace of $\R^\dim$. Otherwise, we can make rows of $A$ orthogonal to $\coef$ and show that $\norm{\cdot}_{\coef}$ is not regular. Second, it also guarantees that the coefficient vector $\coef$ must have some level of sparsity. To see why this is the case, suppose there exists some coordinate $k'$ such that the coefficient $\coef_{k'} \neq 0$ almost surely. If that happens, we can select $A$ such that all of its elements are zero except the $k'$-th row. Then, $\norm{A}_{\coef} = \E |\sum_{j=1}^\dim A_{k',j}\coef_j|\1(\coef_{k'} = 0) = 0,$ but $\norm{A}_{F} > 0$. In this case, the $k'$-th reference coefficient is not sparse at all and the problem becomes ill-conditioned. Third, the regularity of $\norm{\cdot}$ implies that the corresponding dual norm is bounded by the Frobenius norm. Define the dual of $\norm{\cdot}_{\coef}$ to be 
\[
\norm{X}_{\coef}^* = \sup_{A\neq 0, A\in H^\dim}\frac{\tr(X^TA)}{\norm{A}_{\coef}},\quad \text{for~$X\in \R^{\dim\times \dim}$}.
\]
If $\norm{\cdot}_{\coef}$ is $c_{\coef}$-regular, we have $\norm{X}_{\coef}^* \leq \frac{1}{c_{\coef}}\norm{X}_F$, which means its dual semi-norm is upper bounded by the Frobenious norm. This assumption is crucial when we study the local identifiability. As can be seen later in Theorems \ref{Thm:local} and \ref{Thm:region}, regularity of $\norm{\cdot}_{\coef}$ is indispensable in determining the sharpness of the local minimum corresponding to the reference dictionary $\refdict$ as well as the bounding region.}

\paragraph{Assumption II} For any fixed constants $c_1,\ldots, c_\dim\in \R$, the following statement holds almost surely
\[
\sum_{l=1}^\dim c_l\coef_l = 0 \Longrightarrow c_l\coef_l = 0 ~\forall~l=1,\ldots, \dim,
\]
or equivalently, for any fixed $c_1,\ldots, c_\dim$,
\[
P(\sum_{l=1}^\dim c_l\coef_l = 0,  \sum_{l=1}^\dim c_l^2\coef_l^2 > 0) = 0.
\]
Assumption II controls the sparsity of any coefficient vector $\bm \beta$ under a general dictionary $\dict$. For the noiseless signal $\signal = \refdict \coef$, its $j$-th coefficient under a dictionary $\dict$ is $\bm \beta_j = \dict^{-1}[j,]\refdict\coef = \sum_{l=1}^\dim c_l\coef_l$ where $c_l = \dict^{-1}[j,]\refdict_l$ for $l = 1,\ldots, \dim$. Hence the resulting coefficient $\bm \beta_j$ is a linear combination of the reference coefficients $\coef_l$. Thus, Assumption II implies that under any general dictionary, the resulting coefficient $\bm \beta_j$ is zero if and only if for each $l$, either the reference coefficient is zero ($\coef_l = 0$) or the corresponding constant is zero ($c_l = 0$). In other words, elements in the reference coefficient vector cannot `cancel' with each other unless all the elements are zeros. This assumption looks very similar to the \emph{linear independence} property of random variables \citep{Rodgers1984}: Random variables $\psi_1,\ldots, \psi_\dim$ are linearly independent if $c_1\psi_1 + \ldots + c_\dim \psi_\dim = 0$ a.s. implies $c_1=c_2=\cdots=c_\dim = 0$. It is worth pointing out that Assumption II is a weaker assumption than linear independence. Many distributions of interest, such as Bernoulli Gaussian distributions, are not linearly independent but satisfy Assumption II (Proposition \ref{prop:assumption3_is_general}). This assumption will be essential when we study the uniqueness of the sharp local minimum in Theorem \ref{Thm:unique}.

Below we will show that a number of commonly used models satisfy Assumption I and II. 

\paragraph{Bernoulli Gaussian model.} Let $z\in \R^\dim$ be a random vector from standard Gaussian distribution $z\sim \N(0, \I)$. Let $\bm{\xi}\in \{0, 1\}^\dim$ be a random vector whose coordinates are {\it i.i.d.} Bernoulli variables with success probability $\sparsityProb$, i.e., $P(\xi_j = 1) = \sparsityProb$. Define random variable $\coef \in \R^\dim$ to be the element-wise product of $z$ and $\xi$, i.e. $\coef_j = \xi_j z_j$. We say $\coef$ is drawn from Bernoulli Gaussian model with parameter $\sparsityProb$, or $BG(p)$.

\paragraph{Sparse Gaussian model.} Let $z\in \R^\dim$ be a random vector from standard Gaussian distribution $z\sim \N(0, \I)$. Let $\bm{S}$ be a size-s subset uniformly drawn from all size-s subsets of $1,\ldots, \dim$. Let $\bm{\xi}\in \{0, 1\}^\dim$ be a random vector such that $\xi_j = 1$ if $j \in \bm S$ and $\xi_j = 0$ otherwise. Define random variable $\coef \in \R^\dim$ to be the element-wise product of $z$ and $\xi$, i.e., $\coef_j = \xi_j z_j$. We say $\coef$ is drawn from the sparse Gaussian model with parameter $s$, or $SG(s)$.

\vspace*{5mm}

\noindent {\textbf{Remarks:}} Sparse Gaussian and Bernoulli Gaussian distributions have been extensively studied in dictionary learning \citep{Gribonval2010,Wu2015,Schnass2014a,Schnass2014b}. The advantage of using sparse Gaussian and Bernoulli Gaussian distributions is that they are simple and yet able to capture the most important characteristic of the reference coefficients: sparsity. By using sparse Gaussian and Bernoulli Gaussian distributions, \cite{Wu2015} obtains a sufficient and almost necessary condition for local identifiability. Take sparse Gaussian distribution as an example. Let the maximal collinearity $\mu$ of the reference dictionary $\refdict$ be $\mu = \max_{i \neq j}\Big|\big<\refdict_{i}, \refdict_{j}\big>\Big|$ and let $s$ be the sparsity of the reference coefficient vector in the sparse Gaussian model. They show that local identifiability holds when $\mu < \frac{\dim - s}{\sqrt{s}(\dim - 1)}$. From the formula, we can see a trade-off between the maximal collinearity $\mu$ and the sparsity of the coefficient vector $s$. If the coefficient is very sparse, i.e., $s\ll K$, local identifiability holds for a wide range of $\mu$. Otherwise, local identifiability holds for a narrow range of $\mu$. While sparse/Bernoulli Gaussian models can be used to illustrate this trade-off, they are quite restrictive and have little hope to hold for any real data. Several papers \citep{Spielman2013,Arora2014,Arora2014a,Arora2015,Jenatton2014} considered more general models such as sub-Gaussian models. Here, we consider a different  extension of these two models that have not been studied in other papers. Below, we will introduce the Bernoulli type models and exact sparse models.
 
\paragraph{Bernoulli type model $\BG(\sparsityProb_1,\ldots, \sparsityProb_\dim)$.} Let $\bm{z}\in \mathbb{R}^\dim$ be a random vector whose probability density function exists. Let $\bm{\xi}\in \{0, 1\}^\dim$ be a random 0-1 vector. The coordinates of $\bm{\xi}$ are independent and $\xi_j$ is a Bernoulli random variable with success probability $P(\xi_j = 1) = p_j\in (0, 1)$. Define $\bm{\alpha}\in \R^\dim$ such that 
\begin{equation}\label{Model:BG}
\alpha_j = \xi_j z_j ~\forall ~j = 1,\ldots, \dim.
\end{equation}

\paragraph{Exact sparse model $\SG(s)$.} Let $\bm{z}\in \mathbb{R}^\dim$ be a random vector whose probability density function exists. Let $\bm{S}$ be a size-s subset uniformly drawn from all size-s subsets of $1,\ldots, \dim$. Let $\bm{\xi}\in \{0, 1\}^\dim$ be a random variable such that $\xi_j = 1$ if $j \in \bm S$. Define $\coef \in \R^\dim$ such that 
\begin{equation}\label{Model:SG}
\alpha_j = \xi_j z_j ~\forall ~j = 1,\ldots, \dim.
\end{equation}

In the following propositions, we will show that both Bernoulli type models and exact sparse models satisfy Assumption I and Assumption II.

\begin{proposition}\label{prop:specific_constants} The norm induced by exact sparse models or Bernoulli type models satisfy Assumption I. The regularity constant has explicit forms when the coefficient is from $SG(s)$ or $BG(\sparsityProb)$:
\begin{itemize}
    \item If $\coef$ is from $SG(s)$, the norm $\norm{\cdot}_{\coef}$ is $c_{s}$-regular, where $c_s \geq \frac{
    s(\dim - s)}{\dim(\dim - 1)}\sqrt{\frac{2}{\pi}}$.
    \item If $\coef$ is from $BG(\sparsityProb)$, the norm $\norm{\cdot}_{\coef}$ is $c_{p}$-regular, where $c_p\geq p(1-p)\sqrt{\frac{2}{\pi}}$.
\end{itemize}
\end{proposition}






\begin{proposition}\label{prop:assumption3_is_general}
If the coefficient vector is generated from a Bernoulli type model or an exact sparse model, Assumption II holds.
\end{proposition}

\noindent\textbf{Remarks:} (1) Note that our assumptions include many more general distributions beyond sparse/Bernoulli Gaussian models. For example, our model allows $z$ to be from the Laplacian distribution or non-negative, such as Gamma/Beta distributions, see below. The non-negativity of the coefficients breaks the popular expectation condition: $\E \coef_j = 0$, which was used in many previous papers, such as \cite{Gribonval2010,Jenatton2014}. 

\noindent (2) Although our models are quite general, we acknowledge that certain distributions considered in other papers do not satisfy our assumptions. A key requirement in Bernoulli type/exact sparse models is that the probability density function of the base random variable $z$ must exist. For instance, the Bernoulli Randemacher model \citep{Spielman2013} does not satisfy Assumption II. To see this, take the following Bernoulli Randemacher model for $\dim = 2$ as an example: suppose $\bm \xi\in \{0, 1\}^2$ where $P(\xi_1 = 1) = p_1$, $P(\xi_2 = 1) = p_2$. The base random vector $z\in \{-1, 1\}^2$ with $P(z_1 = 1) = P(z_2 = 1) = 1/2$. If we take $c_1 = 1$ and $c_2 = -1$, $P(c_1\coef_1 + c_2 \coef_2 = 0,c_1\coef_1\neq 0,c_2\coef_2\neq 0) = P(\coef_1 - \coef_2 = 0, \xi_1\neq 0, \xi_2\neq 0) = P(\xi_1 = 1, \xi_2 = 1, z_1 = z_2) = p_1\cdot p_2 / 2 > 0$. Therefore, the Assumption II is no longer valid in this case.

Besides sparse/Bernoulli Gaussian models, there are several other models of interest that are special cases of exact sparse/Bernoulli type models. 

\paragraph{Non-negative Sparse Gaussian model.} A random vector $\coef$ is from a non-negative sparse Gaussian model $|SG(s)|$ with parameter $s$ if it is the absolute value of a random vector drawn from $SG(s)$. In other words, for $j = 1,\ldots, \dim$, $\coef_j = |\coef'_j|$ where $\coef'\sim SG(s)$.

\paragraph{Sparse Laplacian model.} Let $z\in \R^\dim$ be a random vector from a standard Laplacian distribution whose probability density function is $f(z) = \frac{1}{2^\dim}\exp(-\|z\|_1).$
Let $\bm{S}$ be a size-s subset uniformly drawn from all size-s subsets of $1,\ldots, \dim$. Let $\bm{\xi}\in \{0, 1\}^\dim$ be a random vector such that $\xi_j = 1$ if $j \in \bm S$ and $\xi_j = 0$ otherwise. Define random variable $\coef \in \R^\dim$ to be the element-wise product of $z$ and $\xi$, i.e., $\coef_j = \xi_j z_j$. We say $\coef$ is drawn from the sparse Laplacian model with parameter $s$, or $SL(s)$.

\section{Main Theoretical Results}\label{Sec:3}

Similar to \cite{Wu2015}, we first study the following optimization problem:


\begin{align}\label{Eq:main2}
\mathrm{minimize~~} & \E \|\inv{\dict} \signal\|_1\\
\mathrm{subject~to~~} & \dict\in \BASIS{\R^\dim}\nonumber
\end{align}


Here, the notation $\E$ is the expectation with respect to $\signal$ under a probabilistic model. Therefore, this optimization problem is equivalent to the case when we have infinite number of samples. As we shall see, the analysis of this population level problem gives us significant insight into the identifiability properties of dictionary learning. We also consider the finite sample case \eqref{Eq:main_finitesample} in Theorem \ref{Thm:global}. 

\subsection{Local identifiability} \label{Sec:3.1}
In this subsection, we establish a sufficient and necessary condition for the reference dictionary to be a sharp local minimum. Theorem \ref{Thm:local} is closely related to the local identifiability result in \cite{Wu2015}, which proves a similar result for sparse/Bernoulli Gaussian models. 


\begin{theorem}[Local identifiability]\label{Thm:local} Suppose the $\ell_1$ norm of the reference coefficient vector $\coef$ has bounded first order moment: $\E \|\coef\|_1 < \infty$. If and only if 
\begin{equation}\label{Eq:local_sharp_condition}
\norm{B(\coef, M^*)}_{\coef}^* < 1,
\end{equation}
$\refdict$ is a sharp local minimum of Formulation \eqref{Eq:main2} with sharpness at least $\frac{c_{\coef}}{\sqrt{2}\norm{\refdict}_2^2}(1 - \norm{B(\coef, M^*)}_{\coef}^*)$.
If $\norm{B(\coef, M^*)}_{\coef}^* > 1,$ $\refdict$ is not a local minimum.
\end{theorem}

\paragraph{Remarks:} \cite{Wu2015} studies the local identifiability problem when the coefficient vector $\coef$ is from Bernoulli Gaussian or sparse Gaussian distributions. They gave a sufficient and almost necessary condition that ensures the reference dictionary to be a local minimum. Theorem \ref{Thm:local} substantially extends their result in two aspects:
\begin{itemize}
    \item The reference coefficient distribution can be exact sparse models and Bernoulli type models, which is more general than sparse/Bernoulli Gaussian models.
    \item In addition to showing that the reference dictionary $\refdict$ is a local minimum, we show that $\refdict$ is actually a sharp local minimum with an explicit bound on the sharpness.
\end{itemize}

To prove Theorem \ref{Thm:local}, we need to calculate how objective function changes along any direction in the neighborhood of the reference dictionary. The major challenge of this calculation is that the objective function is neither convex nor smooth, which prevents us from using sub-gradient or gradient to characterize the its local structure.
We obtain a novel sandwich-type inequality of the $\ell_1$ objective function (Lemma \ref{lemma:3}). With the help of this inequality, we are able to carry out a more fine-grained analysis of the $\ell_1$-minimization objective. The detail proof of Theorem \ref{Thm:local} can be found in \nameref{Sec:Appendix_C}.


\paragraph{Interpretation of the condition \eqref{Eq:local_sharp_condition}:} 
Intuitively, \eqref{Eq:local_sharp_condition} holds if $B(\coef, M^*)$ is small under the dual norm, i.e., $\norm{B(\coef, M^*)}_{\coef}^*<1$. By the definition of $B(\coef, M^*)$, the quantity is the difference between two matrices
\begin{align*}
& B(\coef, M^*) = B_1(\coef) - B_2(\coef, M^*),
\end{align*}
where $(B_1(\coef))_{k, j} = \E \coef_j \sgn(\coef_k)$ and $B_2(\coef, M^*)_{k, j} =  M_{j, k}^* \E |\coef_j|$. Roughly speaking, the first matrix measures the ``correlation" between different coordinates of the coefficients while the second matrix measures the collinearity of the atoms in the reference dictionary. 
For instance, when the coordinates of $\coef$ are independent and mean zero, $B_1(\coef) = 0$. When all atoms in the dictionary are orthogonal, i.e., $M^* = \I$, $B_2(\coef, M^*) = 0$. In that extreme case, the reference dictionary is for sure a local minimum.

\vspace*{5mm}

Theorem \ref{Thm:local} gives the condition under which the reference dictionary is a sharp local minimum. In other words, under that condition, the objective $\E L(\refdict)$ is the smallest within a neighborhood of $\refdict$.
The below Theorem \ref{Thm:region} gives an explicit bound of the size of the region. To the best of authors' knowledge, this is the first result about the region where local identifiability holds for $\ell_1$-minimization.

\begin{theorem}\label{Thm:region}
Under notations in Theorem \ref{Thm:local}, if $\norm{B(\coef, M^*)}_{\coef}^* < 1$, for any $\dict$ in the set $S = \left\{\dict\in \BASIS{\R^\dim}\Big| \norm{\dict}_2 \leq 2\norm{\refdict}_2,
\norm{\dict - \refdict}_F\leq 
\frac{(1 - \norm{B(\coef, M^*)}_{\coef}^*)\cdot  c_{\coef} }{8\sqrt{2}\norm{\refdict}_2^2\max_j \E |\coef_j|}\right\}$, 
we have $\E \|\inv{\dict}\signal\|_1 \geq \E \|\inv{(\refdict)}\signal\|_1$.
\end{theorem}

\paragraph{Remarks:} First of all, note that the set $S$ we study here is different from what \cite{Agarwal2013a} called ``basin of attraction". The basin of attraction of an iterative algorithm is the set of dictionaries such that if the initial dictionary is selected from that set, the algorithm converges to the reference dictionary $\refdict$. For an iterative algorithm that decreases an objective function in each step, its basin of attraction must be a subset of the region within which $\refdict$ has the minimal objective value. Secondly, Theorem \ref{Thm:region} only tells us that $\refdict$ admits the smallest objective function value within the set $S$. It does not, however, indicates that $\refdict$ is the only local minimum within $S$.





\vspace*{5mm}

In what follows, we study two examples to gain a better understanding of the conditions in Theorem \ref{Thm:local} and \ref{Thm:region} . These examples demonstrate the trade-off between coefficient sparsity, collinearity of atoms in the reference dictionary and signal dimension $\dim$. For simplicity, we set the reference dictionary to be the constant collinearity dictionary with coherence $\mu$: $\refdict(\mu) = ((1 - \mu)\mathbb{I} + \mu \mathit{1}\mathit{1}^T)^{1/2}$ ($\mu > 0$). This simple dictionary class was used to illustrate the local identifiability conditions in \cite{Gribonval2010} and \cite{Wu2015}. The coherence parameter $\mu$ controls the collinearity between dictionary atoms. 
By studying this class of reference dictionaries, we can significantly simplify the conditions and demonstrate how the coherence $\mu$ affects dictionary identifiability.

\begin{corollary} \label{coro:4}
Suppose the reference dictionary $\refdict$ is a constant collinearity dictionary with coherence $\mu$: $\refdict(\mu) = ((1 - \mu)\mathbb{I} + \mu \mathit{1}\mathit{1}^T)^{1/2}$ ($\mu > 0$),
and the reference coefficient vector $\coef$ is from $SG(s)$. If and only if
\[
\mu\sqrt{s} < \frac{\dim - s}{\dim - 1},
\]
$\refdict$ is a sharp local minimum with sharpness at least $\frac{s}{\sqrt{\pi}(1 + \mu (\dim - 1))\dim}\left(\frac{\dim - s}{\dim - 1} - \mu \sqrt{s}\right)$. For any $\dict\in S = \left\{\dict\in \BASIS{\R^\dim}\Big| \norm{\dict}_2\leq 2 \sqrt{1 + \mu (\dim - 1)}, \norm{\dict - \refdict}_F\leq \frac{1}{8\sqrt{2}(1 + \mu (\dim - 1))}\left(\frac{\dim - s}{\dim - 1} - \mu \sqrt{s}\right) \right\}$, we have $\E L(\dict) \geq \E L(\refdict)$. 
\end{corollary}
Three parameters play important roles for the reference dictionary to be a sharp local minimum: dictionary atom collinearity $\mu$, sparsity $s$ and dimension $\dim$. Since $\mu\sqrt{s} - \frac{\dim - s}{\dim - 1}$ is a monotonically increasing function with respect to $\mu$ and $s$, local identifiability holds when the dictionary is close to an orthogonal matrix and the coefficient vector is sufficiently sparse. Another important observation is that $\mu\sqrt{s} - \frac{\dim - s}{\dim - 1}$ is monotonically decreasing as $\dim$ increases. That means if the number of nonzero elements per signal $s$ is fixed, the local identifiability condition is easier to satisfy for larger $\dim$. If $\dim$ tends to infinity, the condition becomes $s < \frac{1}{\sqrt{\mu}}$. 
Also, the set $S$ shrinks as $s$ or $\mu$ increases, which means that the region is smaller when the coefficients have less sparsity or the atoms in the dictionary have higher correlations. 
When $\mu = 0$, the set $S$ becomes $ \left\{\dict\in \BASIS{\R^\dim}\Big| \norm{\dict}_2\leq 2, \norm{\dict - \refdict}_F\leq \frac{1}{8\sqrt{2}}\frac{\dim - s}{\dim - 1}\right\} $. 


To illustrate our condition for non-negative coefficient distributions, we consider non-negative sparse Gaussian distributions in the following example. Since we do not have the explicit form of the regularity constant $c_{\coef}$ for this type of distribution, we omit the corresponding results for the sharpness and the region bound.
\begin{corollary}\label{coro:5}
Suppose the reference dictionary is a constant collinearity dictionary with coherence $\mu$: $\refdict(\mu) = ((1 - \mu)\mathbb{I} + \mu \mathit{1}\mathit{1}^T)^{1/2},$
and the reference coefficient vector $\coef$ is from non-negative sparse Gaussian distribution $|SG(s)|$. If 
\[
\Big|\mu - \frac{s - 1}{\dim - 1}\Big| < \frac{\dim - s}{\dim - 1},
\]
$\refdict$ is a sharp local minimum.
\end{corollary}

Note that the condition $\frac{\dim - 1}{\dim - s} \cdot \Big|\mu - \frac{s - 1}{\dim - 1}\Big| < 1$ is equivalent to $\frac{2s - \dim - 1}{\dim - 1}< \mu <1$. When $\dim$ tends to infinity, the reference dictionary is a local minimum for $\mu < 1$. Compared to the same bound from Corollary \ref{coro:4}, $\mu < \frac{1}{\sqrt{s}}$ for large $\dim$, the bound for non-negative coefficients is less restrictive. Therefore, the non-negativeness of the coefficient distribution relaxes the requirement for local identifiability. 


Other interesting examples, such as Bernoulli Gaussian coefficients and sparse Laplacian coefficients, can be found in \nameref{Sec:Appendix_A}. 

\subsection{Global identifiability}\label{Sec:3.2}



For $\ell_1$-minimization, multiple local minima do exist: as a result of sign-permutation ambiguity, if $\dict$ is a local minimum, for any permutation matrix $P$ and any diagonal matrix $\Lambda$ with diagonal elements $\pm 1$, $\dict P \Lambda$ is also a local minimum. These local minima are benign in nature since they essentially refer to the same dictionary. Can there be other local minima than the benign ones? If so, how can we distinguish benign local minima from them? In this subsection, we consider the problem of global identifiability: is the reference dictionary a global minimum? First, we give a counter-example to show that the reference dictionary is not necessary a global minimum of the $\ell_1$-minimization problem even for orthogonal dictionary and sparse coefficients.


\paragraph{Counter-example on global identifiability.} Suppose the reference dictionary is the identity matrix $\mathbb{I}\in \R^{2\times 2}$. The coefficients are generated from a Bernoulli-type model $\coef\in\R^2$ such that $\coef_i = z_i \xi_i$ for $i = 1,2$, where $\xi_1$ and $\xi_2$ are Bernoulli variables with success probability $0.67$, and $(z_1,z_2)$ is drawn from the below Gaussian mixture model: 
\[
\frac{1}{2} \N\left(0,\left(\begin{array}{cc}101& -99\\ -99& 101\end{array}\right)\right) + \frac{1}{2}\N\left(0,\left(\begin{array}{cc}101& 99\\ 99& 101\end{array}\right)\right).
\] 
We generate 2000 samples from the model and compute the dictionary learning objective $L(\cdot)$ defined in \eqref{Eq:L} for each candidate dictionary (Fig. \ref{Fig:counter_example}). As can be seen, the reference dictionary is not the global minimum. 
\begin{figure}[ht]
\centering
\includegraphics[scale=.8]{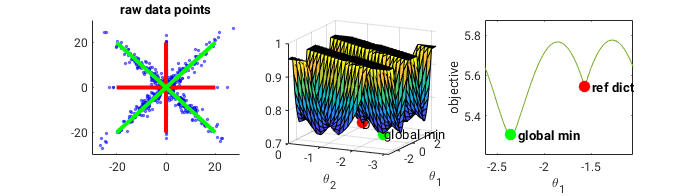}
\caption{The empirical data (Left) and the objective surface plot (Middle). We parameterize a candidate dictionary as $\dict = (a_1, a_2)$, where $a_1 = (\cos(\theta_1), \sin(\theta_1))$, $a_2 = (\cos(\theta_2), \sin(\theta_2))$. The objective of $\dict$ is defined as in \eqref{Eq:L}. Green points indicate global minima, whereas red points are reference dictionary or its sign-permutation equivalents. The Right figure shows the objective curve for all orthogonal dictionaries ($\theta_1 - \theta_2 = \pi/2$). While the reference dictionary is a sharp local minimum, it is not a global minimum.}\label{Fig:counter_example}
\end{figure}

In the above example, although the reference dictionary is not a global minimum, it is still a sharp local minimum and there is no other sharp local minimum. Is this observation true for general cases? The answer is yes. The following theorem shows that the reference dictionary is the unique sharp local minimum of $\ell_1$-minimization up to sign-permutation.

\begin{theorem}[Unique sharp local minimum]\label{Thm:unique}
If $\refdict$ is a sharp local minimum of Formulation \eqref{Eq:main2}, it is the only sharp local minimum in $\BASIS{\R^\dim}$. If it is not a `sharp' local minimum, there will be no sharp local minimum in $\BASIS{\R^\dim}$.
\end{theorem}
Note that Theorem \ref{Thm:unique} works for the population case where the sample size is infinite. For the finite sample case, we obtain a similar result with a stronger assumption on the data generation model.
\begin{theorem}[Asymptotic case]\label{Thm:global} Suppose $n$ samples are drawn i.i.d. from a model satisfying Assumptions I and II, $\|\coef\|_\infty$ is bounded by $L< \infty$,  and $\frac{\dim \ln \dim}{n}\rightarrow 0$, then for any fixed $\epsilon, \rho> 0$, 
\[
P\left(\begin{array}{l}\text{there exists $\dict\neq \refdict$ up to sign-permutation s.t.  $\norm{\dict^{-1}}_2 \leq \rho$} \\
\text{and $\dict$ is a sharp local minimum of \eqref{Eq:main_finitesample} with sharpness at least $\epsilon$}\end{array}\right)\rightarrow 0.
\]
\end{theorem}

\paragraph{Remarks:} Theorem \ref{Thm:global} ensures that \emph{no} dictionaries other than $\refdict$ are sharp local minima within a region $\{\dict\neq \refdict\Big|\norm{\inv{\dict}}_2\leq \rho\}$ if the sample size $n$ is larger than $O(\dim \ln \dim)$. However, it does not tell whether or not $\refdict$ is a sharp local minimum. This latter problem is answered in Theorem \ref{Thm:local}, which gives a sufficient and necessary condition to ensure that the reference dictionary is a sharp local minimum. 


\section{Algorithms for checking sharpness and solving $\ell_1$-minimization}\label{Sec:4}
As shown in the previous section, the reference dictionary is the unique sharp local minimum under mild conditions and certain data generation models. This motivates us to use this property as a stopping criterion for $\ell_1$-minimization. If the algorithm finds a sharp local minimum, we know that it is the reference dictionary. To do so we need to address the following practical questions:
\begin{itemize}
\item How to determine numerically if a given dictionary is a sharp local minimum?
\item How to find a sharp local minimum and recover the reference dictionary?
\end{itemize}
In this section, we will first introduce an algorithm to check if a given dictionary is a sharp local minimum. We will then develop an algorithm that aims at recovering the reference dictionary. The latter algorithm is guaranteed to decrease the (truncated) $\ell_1$ objective function at each iteration (Proposition \ref{prop:monotone}).

\subsection{Determining sharp local minimum} \label{Sec:4.1}
Although the concept of a sharp local minimum is quite intuitive, checking whether a given dictionary is a sharp local minimum can be challenging. First of all, the dimension of the problem is very high ($\dim^2$). Secondly, if a dictionary is a sharp local minimum, the objective function is not differentiable at that point, precluding us from using gradients or Hessian matrix to solve the problem. 

We propose a novel algorithm to address these challenges. We decompose the problem into a series of sub-problems each of which is low-dimensional. In Proposition \ref{prop:sharp_is_stable}, we show that a given dictionary is a sharp local minimum in dimension $\dim^2$ if and only if certain vectors are sharp local minima for the corresponding sub-problems of dimension $\dim$. The objective function of each subproblem is strongly convex. To deal with non-existence of gradient or Hessian matrix, we design a perturbation test based on the observation that a sharp local minimum ought to be stable with respect to small perturbations. For instance, $x = 0$ is the sharp local minimum of $|x|$ but is non-sharp local minimum of $x^2$. If we add a linear function as a perturbation, $x = 0$ is still a local minimum of $|x| + \epsilon \cdot x$ for any $\epsilon$ such that $|\epsilon| < 1$ but not so for $x^2 + \epsilon \cdot x$. The choice of the perturbation is crucial. In Proposition \ref{prop:sharp_is_stable}, we show that adding a perturbation to the dictionary collinearity matrix $M$ is sufficient.  
This leads us to the following theorem:
\begin{proposition}\label{prop:sharp_is_stable}
The following three statements are equivalent:
\begin{itemize}
\item[1)] $\dict$ is a sharp local minimum of \eqref{Eq:main_finitesample}.
\item[2)] For any $k = 1,\ldots, \dim$, $\I_k$ is the sharp local minimum of the strongly convex optimization:
\begin{align}\label{Eq:coordinatewise_sharp}
\I_k \in \mathrm{argmin}_{\bm w}~ & \E |\big<\coef, \bm w\big>| + \sum_{h=1,h\neq k}^\dim \sqrt{(w_{h} - M_{k, h})^2 + 1 - M_{k, h}^2}\cdot\E |\coef_h|.\\ \nonumber
\mathrm{subject~to~}& \bm w = [w_1,\ldots, w_\dim]\in \R^\dim,~ w_k = 1.
\end{align}
\item[3)] For a sufficiently small $\rho > 0$ and any $\tilde M$ s.t. $|\tilde M_{k,h} - M_{k,h}|\leq \rho$ for any $k, h = 1,\ldots, \dim$, $\I_k$ is the local minimum of the convex optimization:
\begin{align}\label{Eq:perturbed_version}
\I_k \in \mathrm{argmin}_{\bm w}~ & \E |\big<\coef, \bm w\big>| + \sum_{h=1,h\neq k}^\dim \sqrt{(w_{h} - \tilde M_{k, h})^2 + 1 - \tilde M_{k, h}^2}\cdot \E |\coef_h|.\\\nonumber
\mathrm{subject~to~}& \bm w = [w_1,\ldots, w_\dim]\in \R^\dim,~ w_k = 1.
\end{align}
for $k = 1,\ldots, \dim$.
\end{itemize}
\end{proposition}
%
Proposition \ref{prop:sharp_is_stable} tells us that, in order to check whether a dictionary is a sharp local minimum, it is sufficient to add a perturbation to the matrix $M = \dict^T\dict$ and check whether the resulting dictionary is the local minimum of the perturbed objective function. Empirically, we can take a small enough $\rho$ and minimize the objective \eqref{Eq:perturbed_version}. If $\I_k$, the $k$-th column vector of the identity matrix, is the local minimum for the perturbed objective, by Proposition \ref{prop:sharp_is_stable} the given dictionary is guaranteed to be a sharp local minimum. We formalize this idea into Algorithm \ref{Alg:test}. We acknowledge that this algorithm might be conservative and classify a sharp local minimum as a non-sharp local minimum if $\rho$ is not small enough as required in Lemma \ref{prop:sharp_is_stable}. There is no good rule-of-thumb in choosing $\rho$ as it has to depend on the specific data. We explore the sensitivity of this algorithm with respect to choice of $\rho$ in the simulation section.

\begin{algorithm}[ht]
\caption{Sharp local minimum test for $\ell_1$-minimization dictionary learning}\label{Alg:test}
\begin{algorithmic}
\REQUIRE Dictionary to be tested $\dict$, samples $\signal^{(1)}, \ldots, \signal^{(n)}$, perturbation level $\rho\in\R^+$, threshold $T\in \R^+$.
\FOR{$i = 1,\ldots, \samplesize$}
\STATE $\bm \beta^{(i)} \gets \inv{\dict}\signal^{(i)}$.
\ENDFOR
\FOR{$j=1,\ldots, \dim$}
\STATE Generate $\epsilon_j \sim \N(0, \rho \cdot \I_{\dim\times\dim})$.
\STATE $\tilde{\dict}_j = \dict_j + \epsilon_j$.
\ENDFOR
\FOR{$k, h = 1, \ldots, \dim$}
\STATE $\tilde{M}_{k, h} \gets \big<\tilde{\dict}_{h},\tilde{\dict}_j\big>$ if $k \neq h$ or 0.
\ENDFOR
\STATE $r \gets 0$
\FOR{$k=1,\ldots, \dim$}
\STATE Solve the strongly convex optimization via BFGS:
\begin{align}\label{Eq:houlai}
\bm w^{(k)}\gets \mathrm{minimize}_{\bm w}~ & \sum_{i = 1}^n |\big<\bm \beta^{(i)}, \bm w\big>| + \sum_{h=1,h\neq k}^{\dim} \sqrt{(w_{h} - \tilde{M}_{h})^2 + 1 - \tilde{m}_{h}^2}\cdot \sum_{i=1}^n |\bm \beta^{(i)}_h|.\\
\mathrm{subject~to~}& \bm w = [w_1,\ldots, w_\dim]\in \R^\dim, \quad w_k = 1.
\end{align}
\STATE $\I_k\gets (0,\ldots, 0, 1, 0, \ldots, 0)$ where only the $k$-th element is 1.
\STATE $r\gets \max(r, \|\bm w^{(k)} - \I_k\|_2^2)$.
\ENDFOR
\IF{$r < T$}
\STATE Output $\dict$ is a sharp local minimum.
\ELSE
\STATE Output $\dict$ is not a sharp local minimum.
\ENDIF
\end{algorithmic}
\end{algorithm}



\vspace*{5mm}

The main component of Algorithm \ref{Alg:test} is solving the strongly convex optimization \eqref{Eq:houlai}. To do so we use Broyden–Fletcher–Goldfarb–Shanno (BFGS) algorithm \citep{Witzgall1989}, which is a second order method that estimates Hessian matrices using past gradient information. Each step of BFGS is of complexity $O(n\dim + \dim^2)$. If we assume the maximum iteration to be a constant, the overall complexity of Algorithm 1 is $O(n\dim^2 + \dim^3)$. Because sample size $n$ is usually larger than the dimension $\dim$, the dominant term in the complexity is $O(n\dim^2)$. In the simulation section, we show that the empirical computation time is in line with the theoretical bound. 




\paragraph{Recovering the reference dictionary.}
We now try to solve formulation \eqref{Eq:main_finitesample}. One of the most commonly used technique in solving dictionary learning is alternating minimization \citep{Olshausen1997,Mairal2009a}, which is to update the coefficients and then the dictionary in a alternating fashion until convergence. This method fails for noiseless $\ell_1$-minimization: when the coefficients are fixed, the dictionary must also be fixed to satisfy all constraints. To allow dictionaries to be updated iteratively, researchers have proposed different ways to relax the constraints \citep{Agarwal2013a, Mairal2014SPAMSV2.5}. However, those workarounds tend to have numerical stability issues if a high precision result is desired \citep{Mairal2014SPAMSV2.5}. This motivates us to propose Algorithm \ref{Alg:bcd}. The algorithm uses the idea from Block Coordinate Descent (BCD). It updates each row of $\inv{\dict}$ and the corresponding row in the coefficient matrix simultaneously. As we update one row of $\inv{\dict}$, we also scale all the other rows of $\inv{\dict}$ by appropriate constants. This is because if we only update one row of $\inv{\dict}$ while keeping the others fixed, columns of the resulting dictionary will not have unit norm. The following lemma gives an admissible parameterization for updating one row of $\inv{\dict}$.

\begin{proposition}\label{prop:param} For any dictionary $\dict\in \BASIS{\R^\dim}$ and any coordinate $k\in 1,\ldots, \dim$, given a vector $w = [w_1,\ldots, w_\dim]\in \R^\dim$ such that $w_k = 1$, we can define a matrix $Q\in \R^{\dim\times \dim}$: 
\begin{align*}
Q[k, ] = & 
\left\{
\begin{array}{ll}
w^T\inv{\dict} & h = k\\
\sqrt{(w_h - M_{k,h})^2 + 1 - M_{k,h}^2}\cdot\inv{\dict}[h, ] &  h \neq k
\end{array}
\right. .
\end{align*}
Then $\inv{Q}\in \BASIS{\R^\dim}$, which means each column of $\inv{Q}$ is of norm 1.
\end{proposition}

With the parameterization in Proposition \ref{prop:param}, we derive the following subproblems from $\ell_1$-minimization dictionary learning: for $k=1,...,K$,
\begin{align*}
\mathrm{argmin}_{\bm w}~ & \sum_{i=1}^n |\big<\beta^{(i)}, \bm w\big>| + \sum_{h=1,h\neq k}^\dim \sqrt{(\bm w_{h} - M_{k, h})^2 + 1 - M_{k, h}^2}\cdot\sum |\beta_h^{(i)}|.\\ \nonumber
\mathrm{subject~to~}& \bm w = [w_1,\ldots, w_\dim]\in \R^\dim, w_k = 1.
\end{align*}
where $\bm \beta^{(i)} = \inv{\dict}\signal^{(i)}$ for a dictionary $\dict$. This new sub-problem is strongly convex, making it relatively easy to solve. We obtain Algorithm \ref{Alg:bcd} by solving this optimization iteratively for each coordinate $k$. Note that the idea of learning a dictionary from solving a series of convex programs has been studied in other papers. \cite{Spielman2013} reformulates the dictionary learning problem as a series of linear programmings (LP) and construct a dictionary from the LP solutions. Nonetheless, their algorithm is not guaranteed to minimize the $\ell_1$ objective at each iteration. 


In our simulation, when the signal-to-noise ratio is high, $\ell_1$-minimization sometimes ends up with a low quality result. This is commonly due to the fact that  the $\ell_1$-norm over-penalizes large coefficients, which breaks the local identifiability, i.e., the reference dictionary is no longer a local minimum. To further enhance the performance of $\ell_1$-minimization, we use ideas similar to re-weighted $\ell_1$ algorithms in the field of compressed sensing \citep{Candes2008}. The motivation of re-weighted algorithms is to reduce the bias of $\ell_1$-minimization by imposing smaller penalty to large coefficients. In our algorithm, we simply truncate coefficient entries beyond a given threshold $\tau$. The obtained problem is still strongly convex but this trick improves the numerical performance significantly.


\begin{algorithm}[ht]
\caption{Dictionary Learning Block Coordinate Descent (DL-BCD)}\label{Alg:bcd}
\begin{algorithmic}
\REQUIRE Data $\signal^{(1)},\ldots, \signal^{(n)}$, threshold $\tau$.
\STATE Initialize $\dict^{(0,1)}$, $t \gets 0$. $Q\gets \inv{(\dict^{(0,1)})}$.
\WHILE{Stopping criterion not satisfied}
\FOR{$j=1,\ldots, \dim$}
\FOR{$i = 1,\ldots, \samplesize$}
\STATE $\bm \beta^{(i)} \gets Q \signal^{(i)}$.
\ENDFOR
\FOR{$h = 1, \ldots, \dim$}
\STATE $m_{h} \gets \big<\dict^{(t, j)}_{h},\dict^{(t, j)}_j\big>$.
\ENDFOR
\STATE Solve the convex optimization via BFGS:
\begin{align*}
\mathrm{minimize}_{\bm w}~ & \sum_{\substack{i=1..\samplesize,\\|\beta^{(i)}_j|<\tau}} |\big<\bm \beta^{(i)}, \bm w\big>| + \sum_{h=1,h\neq j}^\dim \sqrt{(w_{h} - m_{h})^2 + 1 - m_{h}^2}\cdot\sum_{\substack{i=1..\samplesize,\\|\beta^{(i)}_h|<\tau}} |\beta^{(i)}_h|.\\
\mathrm{subject~to~}& \bm w = [w_1,\ldots, w_\dim]\in \R^\dim, ~ w_j = 1.
\end{align*}
\ENDFOR
\STATE Update $j$-th row of $Q$: $Q[j,] \gets Q\bm w$.
\FOR{$h = 1, \ldots, \dim$, $h\neq j$}
\STATE $Q[h,] \gets Q[h, ]\cdot \sqrt{(w_{h} - m_{h})^2 + 1 - m_{h}^2}$.
\ENDFOR
\IF{$j = d$}
\STATE $\dict^{(t+1, 1)}\gets \inv{Q}$.
\ELSE
\STATE $\dict^{(t, j+1)}\gets \inv{Q}$.
\ENDIF
\STATE $t \gets t + 1$.
\ENDWHILE
\end{algorithmic}
\end{algorithm}

The following theorem guarantees that the proposed algorithm always decreases the objective function value. 


\begin{proposition}[Monotonicity]\label{prop:monotone}
Define 
\[f(\dict) = \sum_{i=1}^n \sum_{j=1}^\dim \min\left(\Big|\dict^{-1}[j,] \signal^{(i)} \Big|, \tau\right),\]
where $\tau$ is the threshold used in Algorithm \ref{Alg:bcd}.
Denote by $\dict^{(t,\dim)}$ the dictionary at the $t$-th iteration from Algorithm \ref{Alg:bcd}. $f(\dict^{(t, \dim)})$ decreases monotonically for $t\in \mathbb N$: $f(\dict^{(0,\dim)}) \geq f(\dict^{(1,\dim)}) \geq f(\dict^{(2,\dim)})\ldots$
\end{proposition}


\section{Numerical experiments}\label{Sec:5}

In this section, we evaluate the proposed algorithms with numerical simulations. The code of DL-BCD can be found in the github repository\footnote{https://github.com/shifwang/dl-bcd}. We will study the empirical running time of Algorithm \ref{Alg:test} in the first experiment and examine how the perturbation parameter $\rho$ affects its performance in the second. In the third experiment, we study the sample size requirement for successful recovery of the reference dictionary. In the fourth experiment, we generate the reference dictionary from Gaussian distribution and linear coefficients from sparse Gaussian. We then compare algorithm \ref{Alg:bcd} against other state-of-the-art dictionary learning algorithms \citep{Parker2014,Parker2014a,Parker2016}. 
The first two simulations are not computationally intensive and are carried out on an OpenSuSE OS with Intel(R) Core(TM) i5-5200U CPU 2.20GHz with 12GB memory, while the last two simulations are conducted in a cluster with ~20 cores. 


\subsection{Empirical running time of Algorithm \ref{Alg:test}}\label{Sec:5.1}
We evaluate the empirical computation complexity of Algorithm \ref{Alg:test}. Let the reference dictionary be a constant collinearity dictionary with coherence $\mu = 0.5$, i.e.,
\[
\refdict = (\mathbb{I} + 0.5 \mathit{1}\mathit{1}^T)^{1/2},
\]
The sparse linear coefficients are generated from the Bernoulli Gaussian distribution $BG(p)$ with $p = 0.7$. This specific parameter setting ensures that the reference dictionary is not a local minimum, thus making Algorithm \ref{Alg:test} converge slower. For a fixed dimension, the computation time scales roughly linearly with the sample size, while for fixed sample size, the computation time scales quadratically with dimension $\dim$ (Fig. \ref{Fig:sim1}). That shows the empirical computation complexity of Algorithm \ref{Alg:test} is of order $O(n\dim^2)$, which is consistent with the theoretical complexity. Simulation results remain stable for different parameter settings, see \nameref{Sec:Appendix_B}.

\begin{figure}[ht]
\centering
\includegraphics[scale=.5]{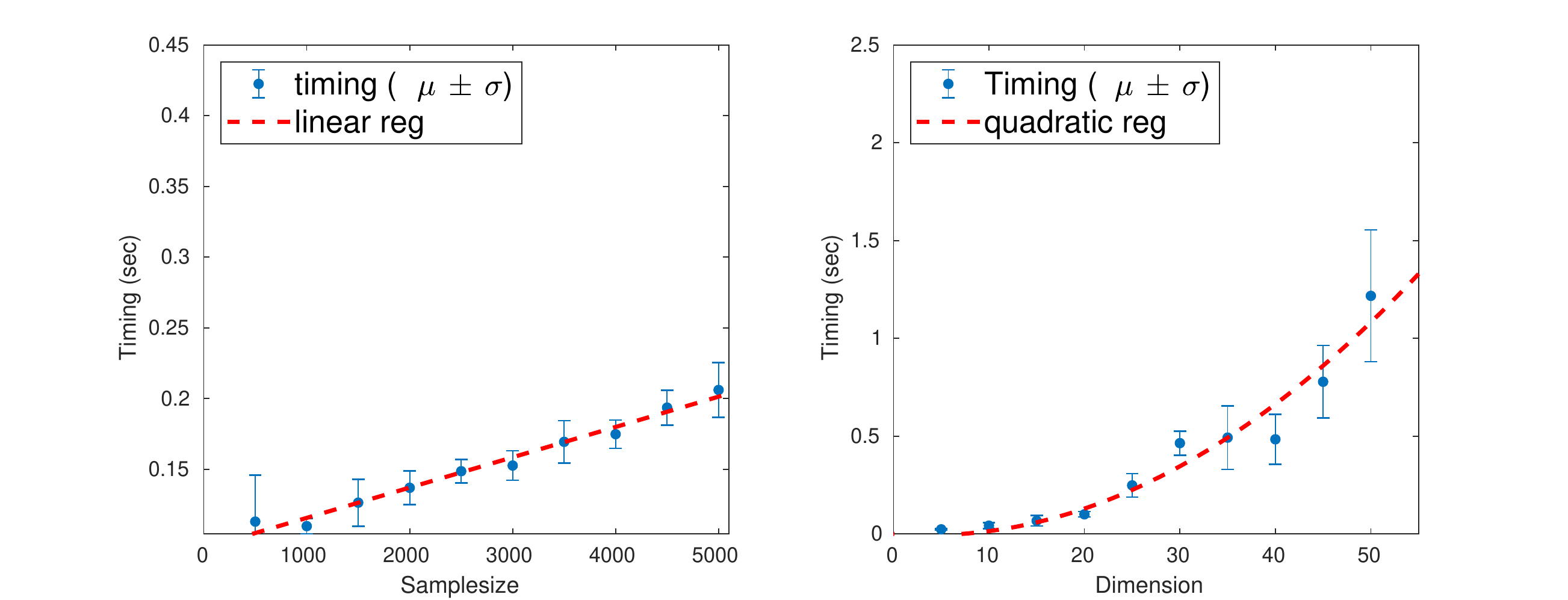}
\caption{Computation time of Algorithm 1. (Left) For $\dim = 20$ and $n=500,\ldots, 5000$. (Right) For $\dim = 5,\ldots,50$ and $n = 400$.}\label{Fig:sim1}
\end{figure}

\subsection{Sensitivity analysis of the perturbation parameter {$\rho$}}\label{Sec:5.2}
In this experiment, we test the sensitivity of Algorithm \ref{Alg:test} by varying the perturbation parameter $\rho$. We set dictionary dimension $\dim = 20$, sparsity parameter $s = 10$ and sample size $n = 1600$. Also, we consider constant collinearity dictionaries with coherence $\mu = \frac{1}{\sqrt{s}}(\frac{\dim - s}{\dim - 1} + 0.1)$ (Fig. \ref{Fig:rho_sensitivity} Left) and $\mu = \frac{1}{\sqrt{s}}(\frac{\dim - s}{\dim - 1} - 0.2)$ (Fig. \ref{Fig:rho_sensitivity} Right). For the first experiment, the reference dictionary is not a sharp local minimum of the objective function given large enough samples. Hence a small perturbation to  the dictionary will result in a large distance $r$ defined in Algorithm \ref{Alg:test}. In the second experiment, the reference dictionary is sharp, indicating the distance $r$ in Algorithm \ref{Alg:test} should be small with respect to perturbation. For each value of $\rho$ between 0.05 and 0.5, we repeat the algorithm 20 times to compute the resulting distances. When $\rho$ is small, the distance $r$ for the non-sharp case is very big (around 1.0) whereas for the sharp case it remains small (around ~$10^{-12}$). For the sharp case, once $\rho$ increases beyond $0.35$, $r$ increases dramatically to ~$10^{-3}$. This experiment shows for a wide range of parameter $\rho$ values ($0.05$ to $0.3$), Algorithm \ref{Alg:test} succeeds in distinguishing between the sharp and not-sharp local minima. Nonetheless, there are two caveats when using this algorithm. Firstly, the parameter $\rho$ depends on the data generation process, which is not known in practice. Thus, it is still an open question about how to select $\rho$. Secondly, this algorithm is only useful when the noise is very small. When the noise is high, the reference dictionary is no longer a sharp local minimum. In that case, instead of checking the sharpness, an alternative would be to check the smallest eigen-value of the Hessian matrix. This idea is not fully explored in this paper and will be studied in future work.
\begin{figure}[ht]
\centering
\includegraphics[width=.8\textwidth]{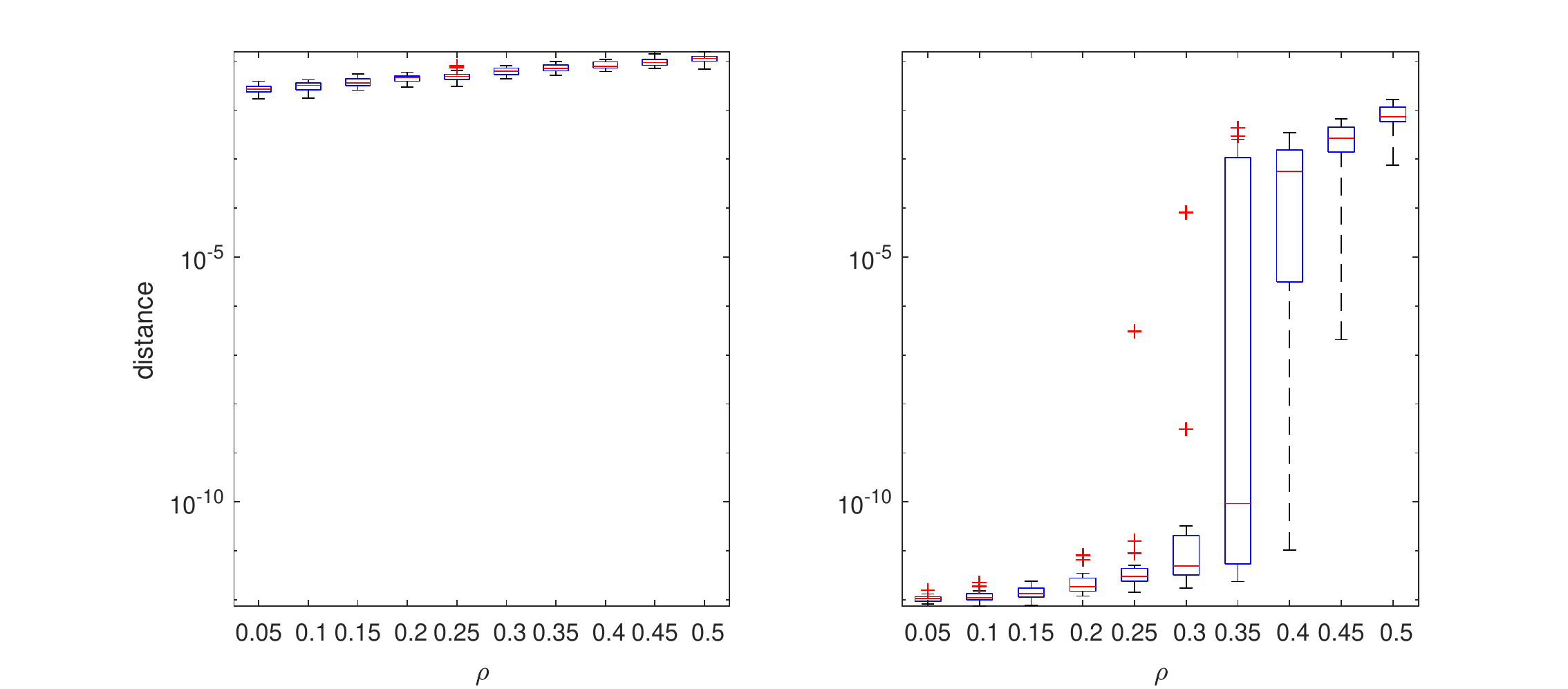}
\caption{Sensitivity analysis of perturbation parameter $\rho$ in Algorithm \ref{Alg:test}. Left: constant collinearity dictionary with coherence $\mu = \frac{1}{\sqrt{s}}(\frac{\dim - s}{\dim - 1} + 0.1)$; Right: constant collinearity dictionary with coherence $\mu = \frac{1}{\sqrt{s}}(\frac{\dim - s}{\dim - 1} - 0.2)$.}\label{Fig:rho_sensitivity}
\end{figure}

\subsection{Empirical sample size requirement for local identifiability}\label{Sec:5.3}
In our analysis, we show that if the sample size $n$ is of order $O(\dim\ln\dim)$, local identifiability will hold with high probability. However, we do not know the corresponding constants that ensure local identifiability. 
In this section, we will study the empirical sample size with the help of Algorithm 1. 

Suppose the reference dictionary has constant coherence $\mu = 0.5$ for various sizes $K=12,16,20$ and the coefficients are drawn from Sparse Gaussian distribution with sparsity $s = 5$. This specific parameter setting ensures the reference dictionary is a sharp local minimum given enough samples. Perturbation level is set at $\rho = 0.01$ and the threshold $T = 10^{-6}$. The experiment is repeated 20 times. The percentage that Algorithm 1 identifies $\refdict$ as a sharp local minimum for different sample size $n$ is shown in Fig. \ref{Fig:sim2}. 

\begin{figure}[h!]
\centering
\includegraphics[width=.8\textwidth]{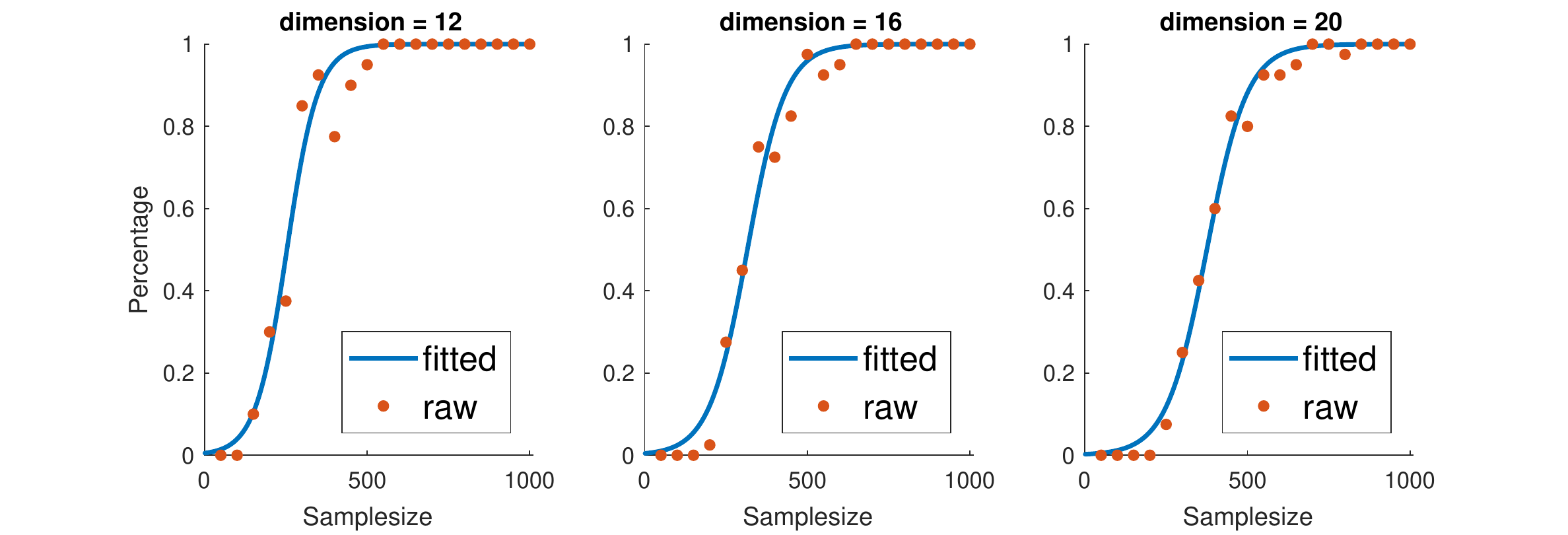}
\caption{The percentage of experiments in which the reference dictionary is a local minimum, for different dimensions $\dim = 12, 16, 20$ and different sample sizes. The fitted line is obtained using a logistic regression. The sample size ensuring 50\% chance is 253, 316, 375 respectively for $\dim = 12, 16, 20$, which is roughly $20\dim$.}\label{Fig:sim2}
\end{figure}
To further explore the required sample size for different dimensions $\dim$, we run simulations for $\dim = 25,...,70$ and estimate the sample complexity that achieves local identifiability with at least 50\% chance, i.e., minimum sample size with percentage $\geq50\%$ in Fig. \ref{Fig:sim2}. As shown in  Fig. \ref{Fig:samplesize_d}, sample complexity and dimension closely follow  a linear relation $16.5\dim + 63$. Note that it is smaller than the $O(\dim \ln \dim)$ because it ensures local identifiability with $50\%$ chance. 
\begin{figure}[!ht]
\centering
\includegraphics[width=.7\textwidth]{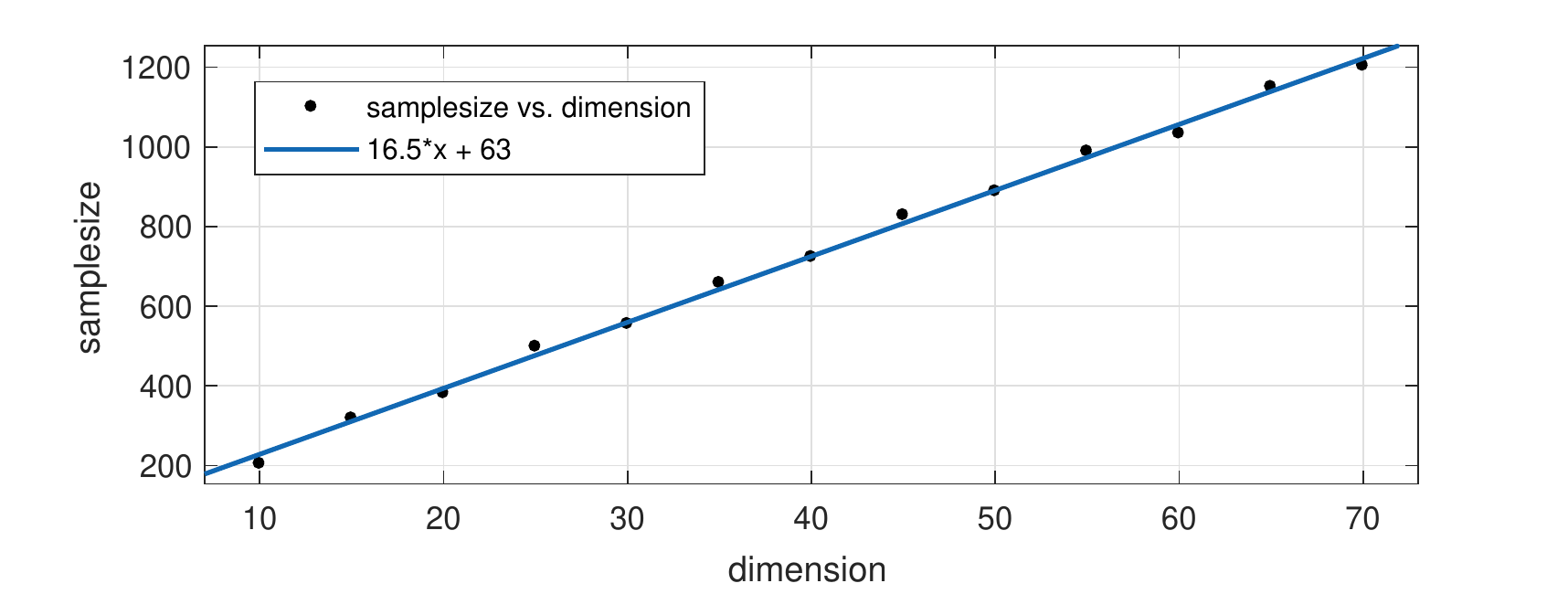}
\caption{The estimated sample size that achieves 50 percent chance to ensure local identifiability for different $\dim$ when the reference coefficient is generated from sparse Gaussian distribution and the reference dictionary has constant collinearity.}\label{Fig:samplesize_d}
\end{figure}
\subsection{Comparison with other algorithms}\label{Sec:5.4}
We compare the performance of DL-BCD with other state-of-the-art algorithms, including the greedy K-SVD algorithm \citep{Aharon2006}, SPAMS for online dictionary learning \citep{Mairal2009,Mairal2009a}, ER-SpUD(proj) for square dictionaries \citep{Spielman2013}, and EM-BiG-AMP algorithm \citep{Parker2014,Parker2014a}. The implementation of these algorithms is available in the MATLAB package BiG-AMP \citep{Parker2014,Parker2014a}.

First we will introduce the simulation setting. We generate $n = 100\dim$ samples using a noisy linear model:
\[
\signal^{(i)} = \refdict \coef^{(i)} + \epsilon^{(i)},\quad i = 1,\ldots, n.
\]
The reference dictionary $\refdict$, the reference coefficients $\coef^{(i)}$, and the noise $\epsilon^{(i)}$ are generated as follows.
\begin{itemize}
    \item Generation of $\refdict$: First, we randomly generate a random Gaussian matrix $X\in \R^{\dim\times\dim}$ where $X_{jk}\sim \N(0,1)$. We then normalize columns of $X$ to construct the columns of the reference dictionary $\refdict_j = X_j/\|X_j\|_2$ for $j = 1,\ldots, \dim$. 
    \item Generation of $\coef^{(i)}$: We generate the reference coefficient from sparse Gaussian distribution with sparsity $s$: $\coef^{(i)}\sim SG(s)$ for $i = 1,\ldots, n$. 
    \item Generation of $\epsilon^{(i)}$: We generate $\epsilon^{(i)}$ using a Gaussian distribution with mean zero. The variance of the distribution is set such that the signal-to--noise ratio is 100:
\[
\frac{\E \|\refdict\coef^{(1)}\|_2}{\E \|\epsilon^{(1)}\|_2} = 10^{2}.
\]
\end{itemize}
We choose the dimension $\dim$ between 2 and 20 and sparsity $s$ between 2 and $\dim$. For each $(s, \dim)$-pair, we repeat the experiment 100 times.
The accuracy of an estimated dictionary $\hat{\dict}$ is quantified using the relative normalized mean square error (NMSE):
\[
\mathrm{NMSE}(\hat{\dict}, \refdict) = \min_{J\in \mathcal{J}} \frac{\|\hat{\dict}J - \refdict\|_F^2}{\|\refdict\|_F^2},
\]
where $\mathcal{J} = \{\Gamma \cdot \Lambda \mid \Gamma \text{~is~a~permutation~matrix~and~}\Lambda\text{~is~a~diagonal~matrix~whose~diagonal~elements~are~}\pm 1.\}$ is a set introduced to resolve the permutation and scale ambiguities.
We say an algorithm has a successful recovery if the NMSE of $\hat{\dict}$ is smaller than the threshold 0.01. 
We compare different algorithms in terms of their recovery rate, defined as the proportion of simulations that an algorithm has a successful recovery.  

The algorithms being tested have several important parameters. For the purpose of comparison, we choose these parameters in a way such that they are consistent with other papers \citep{Parker2014, Parker2014a}. The details of parameter settings can be found in Appendix D. 

Figure \ref{Fig:phase_tran} shows the recovery rate for a variety of choices of dimension $\dim$ and sparsity $s$. For each algorithm, the blue region corresponds to $(s,\dim)$ configurations under which an algorithm has high recovery rate, whereas yellow region indicates low recovery rate. Our results demonstrate that  DL-BCD with $\tau = .5$ has the best recovery performance compared to other algorithms. Note that we select $\tau$ by hand but without much fine tuning. The algorithm EM-BiG-AMP has the second best performance. 

We also compare the algorithms in terms of their computation cost. We record the average computation times for $\dim = 20$ and $s = 10$ (Figure \ref{Fig:time}). It can be seen that the SPAMS package is the fastest. The speed of our DL-BCD is roughly the same as that of K-SVD. ER-SpUD is the slowest among all the algorithms. 

\begin{figure}[ht]
\centering
\includegraphics[width=.7\textwidth]{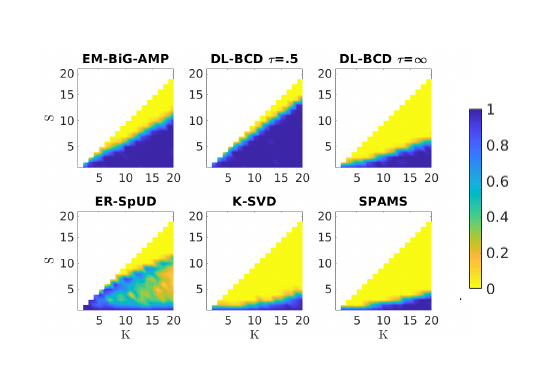}
\vspace{-10mm}
\caption{ Recovery rate of different algorithms for $\dim = 2,\ldots, 20$ and $s = 2,\ldots, \dim$.}\label{Fig:phase_tran}
\end{figure}


\begin{figure}[ht]
    \centering
    \includegraphics[scale=.5]{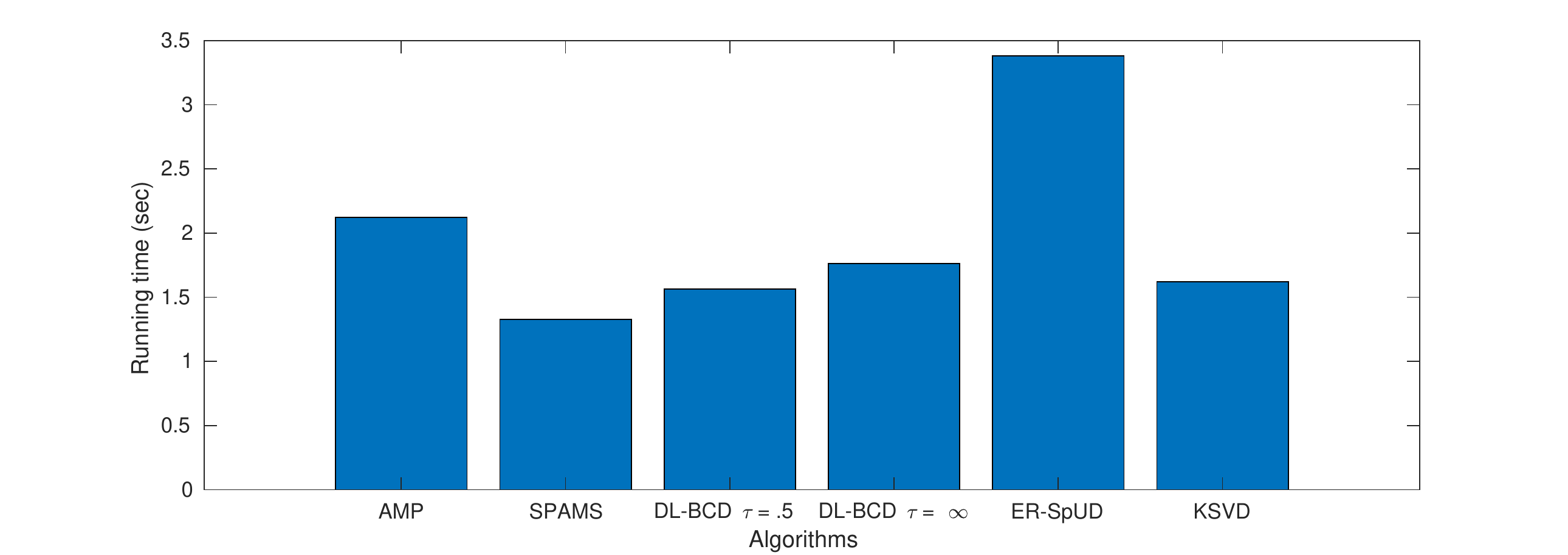}
    \caption{Average running time of different algorithms for $\dim=20$ and $s=10$.}
    \label{Fig:time}
\end{figure}

\section{Conclusions and future work}\label{Sec:6}
In this paper, we study the theoretical properties of $\ell_1$-minimization dictionary learning under complete reference dictionary and noiseless signal assumptions. First, we derive a sufficient and almost necessary condition of local identifiability {of $\ell_1$-minimization}. Our theorems not only extend previous local identifiability results to a much wider class of coefficient distributions, but also give an explicit bound on the region within which the objective value of the reference dictionary is minimal and characterize the sharpness of a local minimum. Secondly, we show that the reference dictionary is the unique sharp local minimum for $\ell_1$-minimization. Based on our theoretical results, we design an algorithm to check the sharpness of a local minimum numerically. Finally, We propose the DL-BCD algorithm and demonstrate its competitive performance over other state-of-the-art algorithms in noiseless complete dictionary learning. 
Future works include generalization of the results to the over-complete or noisy case and as well as other $\ell_1$ type objectives.

\section*{Acknowledgments}
The authors would like to thank Tanya Veeravalli, Simon Walter, Raaz Dwivedi, and Zihao Chen from University of California, Berkeley for their very helful comments of this paper that greatly improve its presentation. 
Partial supports are gratefully acknowledged from ARO grant
W911NF1710005, ONR grant N00014-16-1-2664, NSF grants
DMS-1613002 and IIS 1741340, and the Center for Science of
Information (CSoI), a US NSF Science and Technology Center,
under grant agreement CCF-0939370.
\bibliography{papers}

\begin{thebibliography}{46}
\providecommand{\natexlab}[1]{#1}
\providecommand{\url}[1]{\texttt{#1}}
\expandafter\ifx\csname urlstyle\endcsname\relax
  \providecommand{\doi}[1]{doi: #1}\else
  \providecommand{\doi}{doi: \begingroup \urlstyle{rm}\Url}\fi

\bibitem[Agarwal et~al.(2013)Agarwal, Anandkumar, and Netrapalli]{Agarwal2013}
Alekh Agarwal, Animashree Anandkumar, and Praneeth Netrapalli.
\newblock {A Clustering Approach to Learn Sparsely-Used Overcomplete
  Dictionaries}.
\newblock \emph{arXiv preprint}, arxiv:1309\penalty0 (1952):\penalty0 1--31,
  2013.
\newblock ISSN 0018-9448.
\newblock \doi{10.1109/TIT.2016.2614684}.
\newblock URL \url{http://arxiv.org/abs/1309.1952}.

\bibitem[Agarwal et~al.(2014)Agarwal, Anandkumar, Jain, and
  Netrapalli]{Agarwal2013a}
Alekh Agarwal, Animashree Anandkumar, Prateek Jain, and Praneeth Netrapalli.
\newblock {Learning Sparsely Used Overcomplete Dictionaries via Alternating
  Minimization}.
\newblock \emph{SIAM Journal on Optimization}, 26\penalty0 (4):\penalty0
  2775--2799, 2014.

\bibitem[Aharon et~al.(2005)Aharon, Elad, and Bruckstein]{Aharon2005}
Michal Aharon, Michael Elad, and Alfred~M Bruckstein.
\newblock {K-SVD and its Non-Negative Variant for Dictionary Design}.
\newblock \emph{International Society for Optics and Photonics}, 5914:\penalty0
  591411, 2005.
\newblock URL
  \url{http://www.ejercitodelaire.mde.es/EA/ejercitodelaire/es/aeronaves/avion/Airbus-A400M-T.23/}.

\bibitem[Aharon et~al.(2006)Aharon, Elad, and Bruckstein]{Aharon2006}
Michal Aharon, Michael Elad, and Alfred Bruckstein.
\newblock {K-SVD: An algorithm for designing overcomplete dictionaries for
  sparse representation}.
\newblock \emph{IEEE Transactions on Signal Processing}, 54\penalty0
  (11):\penalty0 4311--4322, 2006.
\newblock ISSN 1053587X.
\newblock \doi{10.1109/TSP.2006.881199}.

\bibitem[Arora et~al.(2014{\natexlab{a}})Arora, Bhaskara, Ge, and
  Ma]{Arora2014}
Sanjeev Arora, Aditya Bhaskara, Rong Ge, and Tengyu Ma.
\newblock {More algorithms for provable dictionary learning}.
\newblock \emph{arXiv preprint arXiv:1401.0579}, page~23, 2014{\natexlab{a}}.
\newblock URL \url{http://arxiv.org/abs/1401.0579}.

\bibitem[Arora et~al.(2014{\natexlab{b}})Arora, Ge, and Moitra]{Arora2014a}
Sanjeev Arora, Rong Ge, and Ankur Moitra.
\newblock {New algorithms for learning incoherent and overcomplete
  dictionaries}.
\newblock In \emph{Conference on Learning Theory}, pages 779--806,
  2014{\natexlab{b}}.
\newblock URL \url{https://arxiv.org/pdf/1308.6273.pdf}.

\bibitem[Arora et~al.(2015)Arora, Ge, Ma, and Moitra]{Arora2015}
Sanjeev Arora, Rong Ge, Tengyu Ma, and Ankur Moitra.
\newblock {Simple, Efficient, and Neural Algorithms for Sparse Coding}.
\newblock \emph{arXiv:1503.00778 [cs, stat]}, 2015.
\newblock ISSN 15337928.
\newblock URL
  \url{http://arxiv.org/abs/1503.00778{\%}5Cnhttp://www.arxiv.org/pdf/1503.00778.pdf}.

\bibitem[Barak et~al.(2014)Barak, Kelner, and Steurer]{Barak2014}
Boaz Barak, Jonathan~A Kelner, and David Steurer.
\newblock {Dictionary Learning and Tensor Decomposition via the Sum-of-Squares
  Method}.
\newblock \emph{Proceedings of the Forty-seventh Annual ACM Symposium on Theory
  of Computing}, pages 143--151, 2014.
\newblock ISSN 07378017.
\newblock \doi{10.1145/2746539.2746605}.
\newblock URL \url{http://arxiv.org/abs/1407.1543
  http://doi.acm.org/10.1145/2746539.2746605}.

\bibitem[Brunet et~al.(2004)Brunet, Tamayo, Golub, and Mesirov]{Brunet2004}
Jean-Philippe Brunet, Pablo Tamayo, Todd~R. Golub, and Jill~P. Mesirov.
\newblock {Metagenes and molecular pattern discovery using matrix
  factorization}.
\newblock \emph{Proceedings of the National Academy of Sciences}, 101\penalty0
  (12):\penalty0 4164--4169, 2004.
\newblock ISSN 0027-8424.
\newblock \doi{10.1073/pnas.0308531101}.
\newblock URL \url{http://www.ncbi.nlm.nih.gov/pubmed/15016911}.

\bibitem[Candes et~al.(2008)Candes, Wakin, and Boyd]{Candes2008}
Emmanuel~J. Candes, Michael~B. Wakin, and Stephen~P. Boyd.
\newblock {Enhancing sparsity by reweighted L1 minimization}.
\newblock \emph{Journal of Fourier analysis and applications}, 14\penalty0
  (5):\penalty0 877--905, 2008.

\bibitem[Chatterji and Bartlett(2017)]{Chatterji2017}
Niladri~S. Chatterji and Peter~L. Bartlett.
\newblock {Alternating minimization for dictionary learning with random
  initialization}.
\newblock In \emph{Advances in Neural Information Processing Systems}, pages
  1994--2003, 2017.
\newblock URL \url{http://arxiv.org/abs/1711.03634}.

\bibitem[Dhara and Dutta(2011)]{Dhara2011}
Anulekha Dhara and Joydeep Dutta.
\newblock \emph{{Optimality conditions in convex optimization: a
  finite-dimensional view}}.
\newblock CRC Press, 2011.

\bibitem[Elad and Aharon(2006)]{Elad2006}
Michael Elad and Michal Aharon.
\newblock {Image denoising via sparse and redundant representations over
  learned dictionaries}.
\newblock \emph{Image Processing, IEEE Transactions on}, 15\penalty0
  (12):\penalty0 3736--3745, 2006.

\bibitem[Ge et~al.(2016)Ge, Lee, and Ma]{Ge2017}
Rong Ge, Jason~D. Lee, and Tengyu Ma.
\newblock {Matrix Completion has No Spurious Local Minimum}.
\newblock \emph{Proceedings of the 30th International Conference on Neural
  Information Processing Systems}, pages 1--27, 2016.
\newblock ISSN 10495258.
\newblock URL \url{http://arxiv.org/abs/1605.07272
  http://dl.acm.org/citation.cfm?id=3157382.3157431}.

\bibitem[Geng et~al.(2014)Geng, Wright, and Wang]{Geng2014}
Quan Geng, John Wright, and Huan Wang.
\newblock {On the local correctness of L1-minimization for dictionary
  learning}.
\newblock In \emph{Information Theory (ISIT), 2014 IEEE International Symposium
  on}, pages 3180--3184. IEEE, 2014.

\bibitem[Gribonval and Schnass(2010)]{Gribonval2010}
R{\'{e}}mi Gribonval and Karin Schnass.
\newblock {Dictionary Identification - Sparse Matrix-Factorisation via
  L1-Minimisation}.
\newblock \emph{Information Theory, IEEE Transactions on}, 56\penalty0
  (7):\penalty0 3523--3539, 2010.

\bibitem[Hoyer(2002)]{Hoyer2002}
Patrik~O. Hoyer.
\newblock {Non-negative sparse coding}.
\newblock \emph{Neural Networks for Signal Processing - Proceedings of the IEEE
  Workshop}, 2002-Janua:\penalty0 557--565, 2002.
\newblock ISSN 0780376161.
\newblock \doi{10.1109/NNSP.2002.1030067}.

\bibitem[Jenatton et~al.(2014)Jenatton, Bach, and Gribonval]{Jenatton2014}
Rodolphe Jenatton, F~Bach, and R~Gribonval.
\newblock {Sparse and spurious: dictionary learning with noise and outliers}.
\newblock \emph{arXiv preprint arXiv:1407.5155}, 2014.

\bibitem[Kreutz-delgado et~al.(2003)Kreutz-delgado, Murray, Rao, Engan, Lee,
  and Sejnowski]{Kreutz-Delgado2003}
Kenneth Kreutz-delgado, Joseph~F. Murray, Bhaskar~D. Rao, Kjersti Engan, Te-Won
  Lee, and Terrence~J. Sejnowski.
\newblock {Dictionary learning algorithms for sparse representation.}
\newblock \emph{Neural computation}, 15\penalty0 (2):\penalty0 349--96, 2003.
\newblock ISSN 0899-7667.
\newblock \doi{10.1162/089976603762552951}.
\newblock URL
  \url{http://www.pubmedcentral.nih.gov/articlerender.fcgi?artid=2944020{\&}tool=pmcentrez{\&}rendertype=abstract}.

\bibitem[Lee and Seung(2001)]{Lee2001}
Daniel Lee and Sebastian Seung.
\newblock {Algorithms for Non-negative Matrix Factorization}.
\newblock In \emph{Advances in Neural Information Processing Systems 13}, 2001.
\newblock ISBN 9781424418206.
\newblock \doi{10.1109/IJCNN.2008.4634046}.

\bibitem[Lesage et~al.(2005)Lesage, Gribonval, Bimbot, and
  Benaroya]{Lesage2005}
Sylvain Lesage, R{\'{e}}mi Gribonval, Fr{\'{e}}d{\'{e}}ric Bimbot, and Laurent
  Benaroya.
\newblock {Learning unions of orthonormal bases with thresholded singular value
  decomposition}.
\newblock In \emph{Acoustics, Speech, and Signal Processing, 2005.
  Proceedings.(ICASSP'05). IEEE International Conference on}, volume~5, pages
  v----293. IEEE, IEEE, 2005.
\newblock \doi{10.1109/ICASSP.2005.1416298>}.
\newblock URL \url{https://hal.inria.fr/inria-00564483}.

\bibitem[Mairal et~al.(2009{\natexlab{a}})Mairal, Bach, Ponce, and
  Sapiro]{Mairal2009a}
Julien Mairal, Francis Bach, Jean Ponce, and Guillermo Sapiro.
\newblock {Online Dictionary Learning for Sparse Coding}.
\newblock \emph{Proceedings of the 26th Annual International Conference on
  Machine Learning}, pages 689--696, 2009{\natexlab{a}}.

\bibitem[Mairal et~al.(2009{\natexlab{b}})Mairal, Bach, Ponce, Sapiro, and
  Zisserman]{Mairal2009}
Julien Mairal, Francis~R. Bach, Jean Ponce, Guillermo Sapiro, and Andrew
  Zisserman.
\newblock {Supervised dictionary learning}.
\newblock In \emph{Advances in neural information processing systems}, pages
  1033--1040, 2009{\natexlab{b}}.

\bibitem[Mairal et~al.(2014)Mairal, Jenatton, Bach, Ponce, Obozinski, Yu,
  Sapiro, and Harchaoui]{Mairal2014SPAMSV2.5}
Julien Mairal, Rodolphe Jenatton, Francis Bach, Jean Ponce, Guillaume
  Obozinski, Bin Yu, Guillermo Sapiro, and Zaid Harchaoui.
\newblock {SPAMS : a SPArse Modeling Software , v2.5}, 2014.

\bibitem[Olshausen and Field(1996)]{Olshausen1996}
Bruno~A Olshausen and David~J Field.
\newblock {Emergence of simple-cell receptive field properties by learning a
  sparse code for natural images}.
\newblock \emph{Nature}, 381\penalty0 (6583):\penalty0 607--609, 1996.

\bibitem[Olshausen and Field(1997)]{Olshausen1997}
Bruno~A. Olshausen and David~J. Field.
\newblock {Sparse coding with an overcomplete basis set: A strategy employed by
  V1?}
\newblock \emph{Vision research}, 37\penalty0 (23):\penalty0 3311--3325, 1997.
\newblock ISSN 00426989.
\newblock \doi{10.1016/S0042-6989(97)00169-7}.

\bibitem[Parker and Schniter(2016)]{Parker2016}
Jason~T. Parker and Philip Schniter.
\newblock {Parametric Bilinear Generalized Approximate Message Passing}.
\newblock In \emph{IEEE Journal on Selected Topics in Signal Processing},
  volume~10, pages 795--808, 2016.
\newblock ISBN 1053-587X.
\newblock \doi{10.1109/JSTSP.2016.2539123}.

\bibitem[Parker et~al.(2014{\natexlab{a}})Parker, Schniter, and
  Cevher]{Parker2014}
Jason~T Parker, Philip Schniter, and Volkan Cevher.
\newblock {Bilinear generalized approximate message passing—Part I:
  Derivation}.
\newblock \emph{IEEE Transactions on Signal Processing}, 62\penalty0
  (22):\penalty0 5839--5853, 2014{\natexlab{a}}.
\newblock URL
  \url{http://www2.ece.ohio-state.edu/{~}schniter/BiGAMP/BiGAMP.html}.

\bibitem[Parker et~al.(2014{\natexlab{b}})Parker, Schniter, and
  Cevher]{Parker2014a}
Jason~T Parker, Philip Schniter, and Volkan Cevher.
\newblock {Bilinear generalized approximate message passing—Part II:
  Applications}.
\newblock \emph{IEEE Transactions on Signal Processing}, 62\penalty0
  (22):\penalty0 5854--5867, 2014{\natexlab{b}}.
\newblock URL
  \url{http://www2.ece.ohio-state.edu/{~}schniter/BiGAMP/BiGAMP.html}.

\bibitem[Peyr{\'e}(2009)]{Peyre2009}
Gabriel Peyr{\'e}.
\newblock Sparse modeling of textures.
\newblock \emph{Journal of Mathematical Imaging and Vision}, 34\penalty0
  (1):\penalty0 17--31, 2009.

\bibitem[Polyak(1979)]{polyak79}
B~T Polyak.
\newblock {Sharp Minima}.
\newblock In \emph{Institute of Control Sciences Lecture Notes, Moscow IIASA
  Workshop On Generalized Lagrangians and Their Applications, IIASA, Laxenburg,
  Austria}, 1979.

\bibitem[Qiu et~al.(2014)Qiu, Patel, and Chellappa]{Qiu2014}
Qiang Qiu, Vishal~M Patel, and Rama Chellappa.
\newblock {Information-theoretic Dictionary Learning for Image Classification}.
\newblock \emph{Pattern Analysis and Machine Intelligence, IEEE Transactions
  on}, 36\penalty0 (11):\penalty0 2173--2184, 2014.

\bibitem[Rodgers et~al.(1984)Rodgers, Nicewander, and Toothaker]{Rodgers1984}
Joseph~Lee Rodgers, W.~Alan Nicewander, and Larry Toothaker.
\newblock {Linearly independent, orthogonal, and uncorrelated variables}.
\newblock \emph{American Statistician}, 38\penalty0 (2):\penalty0 133--134,
  1984.
\newblock ISSN 15372731.
\newblock \doi{10.1080/00031305.1984.10483183}.

\bibitem[Rubinstein et~al.(2010)Rubinstein, Bruckstein, and
  Elad]{Rubinstein2010}
Ron Rubinstein, Alfred~M. Bruckstein, and Michael Elad.
\newblock {Dictionaries for sparse representation modeling}.
\newblock \emph{Proceedings of the IEEE}, 98\penalty0 (6):\penalty0 1045--1057,
  2010.
\newblock ISSN 00189219.
\newblock \doi{10.1109/JPROC.2010.2040551}.

\bibitem[Schnass(2014)]{Schnass2014b}
Karin Schnass.
\newblock {On the identifiability of overcomplete dictionaries via the
  minimisation principle underlying K-SVD}.
\newblock \emph{Applied and Computational Harmonic Analysis}, 37\penalty0
  (3):\penalty0 464--491, 2014.
\newblock ISSN 1096603X.
\newblock \doi{10.1016/j.acha.2014.01.005}.

\bibitem[Schnass(2015)]{Schnass2014a}
Karin Schnass.
\newblock {Local Identification of Overcomplete Dictionaries}.
\newblock \emph{Journal of Machine Learning Research}, 16:\penalty0 1211--1242,
  2015.
\newblock ISSN 15337928.
\newblock URL \url{http://arxiv.org/abs/1401.6354}.

\bibitem[Spielman et~al.(2013)Spielman, Wang, and Wright]{Spielman2013}
Daniel~A. Spielman, Huan Wang, and John Wright.
\newblock {Exact recovery of sparsely-used dictionaries}.
\newblock In \emph{Proceedings of the Twenty-Third international joint
  conference on Artificial Intelligence}, pages 3087--3090. AAAI Press, 2013.
\newblock URL \url{https://arxiv.org/pdf/1206.5882.pdf}.

\bibitem[Srebro and Jaakkola(2003)]{Srebro2003}
Nathan Srebro and Tommi Jaakkola.
\newblock {Weighted low-rank approximations}.
\newblock In \emph{Proceedings of the Twentieth International Conference on
  Machine Learning}, volume~3, pages 720--727, 2003.
\newblock ISBN 1-57735-189-4.
\newblock \doi{10.1.1.5.3301}.
\newblock URL
  \url{http://citeseerx.ist.psu.edu/viewdoc/summary?doi=10.1.1.5.3301}.

\bibitem[Sun et~al.(2017{\natexlab{a}})Sun, Qu, and Wright]{Sun2017a}
Ju~Sun, Qing Qu, and John Wright.
\newblock {Complete Dictionary Recovery Over the Sphere I: Overview and the
  Geometric Picture}.
\newblock \emph{IEEE Transactions on Information Theory}, 63\penalty0
  (2):\penalty0 853--884, 2017{\natexlab{a}}.
\newblock ISSN 00189448.
\newblock \doi{10.1109/TIT.2016.2632162}.

\bibitem[Sun et~al.(2017{\natexlab{b}})Sun, Qu, and Wright]{Sun2017b}
Ju~Sun, Qing Qu, and John Wright.
\newblock {Complete Dictionary Recovery Over the Sphere II: Recovery by
  Riemannian Trust-Region Method}.
\newblock \emph{IEEE Transactions on Information Theory}, 63\penalty0
  (2):\penalty0 885--914, 2017{\natexlab{b}}.
\newblock ISSN 00189448.
\newblock \doi{10.1109/TIT.2016.2632149}.

\bibitem[Tibshirani(1996)]{Tibshirani1991}
Robert Tibshirani.
\newblock {Regression shrinkage and selection via the Lasso}.
\newblock \emph{Journal of the Royal Statistical Society. Series B
  (Methodological)}, 1996.
\newblock ISSN 0035-9246.
\newblock \doi{10.2307/2346101}.

\bibitem[Wainwright(2019)]{Martin2017}
Martin Wainwright.
\newblock \emph{{High-dimensional statistics: A non-asymptotic viewpoint}}.
\newblock Cambridge University Press, 2019.
\newblock ISBN 978-1108498029.
\newblock URL
  \url{https://www.amazon.com/High-Dimensional-Statistics-Non-Asymptotic-Statistical-Probabilistic/dp/1108498027}.

\bibitem[Witzgall and Fletcher(1989)]{Witzgall1989}
Christoph Witzgall and R.~Fletcher.
\newblock {Practical Methods of Optimization.}
\newblock \emph{Mathematics of Computation}, 1989.
\newblock ISSN 00255718.
\newblock \doi{10.2307/2008742}.

\bibitem[Wu and Yu(2018)]{Wu2015}
Siqi Wu and Bin Yu.
\newblock {Local identifiability of l 1 -minimization dictionary learning: a
  sufficient and almost necessary condition}.
\newblock \emph{Journal of Machine Learning Research}, 18:\penalty0 1--56,
  2018.
\newblock URL \url{http://jmlr.org/papers/volume18/16-119/16-119.pdf}.

\bibitem[Wu et~al.(2016)Wu, Joseph, Hammonds, Celniker, Yu, and Frise]{Wu2016}
Siqi Wu, Antony Joseph, Ann~S. Hammonds, Susan~E. Celniker, Bin Yu, and Erwin
  Frise.
\newblock {Stability-driven nonnegative matrix factorization to interpret
  spatial gene expression and build local gene networks}.
\newblock \emph{Proceedings of the National Academy of Sciences}, 113\penalty0
  (16):\penalty0 201521171, 2016.
\newblock ISSN 0027-8424.
\newblock \doi{10.1073/pnas.1521171113}.
\newblock URL \url{http://www.pnas.org/content/113/16/4290.full}.

\bibitem[Zibulevsky and Pearlmutter(2001)]{Zibulevsky2001}
Michael Zibulevsky and Barak~A. Pearlmutter.
\newblock {Blind Source Separation by Sparse Decomposition in a Signal
  Dictionary}.
\newblock \emph{Neural computation}, 4\penalty0 (13):\penalty0 863--882, 2001.

\end{thebibliography}
\newpage
\section*{Appendix A: Additional Examples}
\label{Sec:Appendix_A}
\begin{corollary}\label{coro:appendix_1}
If the reference dictionary is a constant collinearity dictionary with coherence $\mu$ and the coefficients are generated from Bernoulli Gaussian distribution $BG(p)$, if $\frac{\mu\sqrt{p(\dim - 1)}}{1-p}< 1$, then the reference dictionary is a sharp local minimum with sharpness at least $\frac{p}{\sqrt{\pi}(1 + \mu(\dim - 1))}\left(1 - p - \mu\sqrt{p(K-1)}\right)$. For any $$\dict \in \left\{\dict\in \BASIS{\R^\dim}\Big| \norm{\dict}_2 \leq 2\sqrt{1 + \mu(\dim - 1)}, \norm{\dict - \refdict}_F^2 \leq \frac{1}{8\sqrt{2}(1 + \mu(\dim - 1))}\left(1 - p - \mu\sqrt{p(K-1)}\right) \right\},$$

$\E \|\inv{\dict}\signal\|_1 \geq \E \|\coef\|_1.$
\end{corollary}

\begin{corollary}\label{coro:appendix_2}
If the reference dictionary is a constant collinearity dictionary with coherence $\mu$ and the coefficients are generated from sparse Laplacian distribution $SL(s)$, if 
\[
\frac{\mu s(\dim - 1)}{(\dim - s) \iint_{0}^\infty |y - x|(xy)^{s-1}\exp(-(x+y))\Gamma(s)^{-2}dxdy} < 1,
\]
then the reference dictionary is a sharp local minimum. 
\end{corollary}
When $\dim = 10$ and $20$, the phase transition curve ($\norm{B(\coef, M^*)}_{\coef}^* = 1$) for sparse Laplace distribution and sparse Gaussian distribution can be found in Fig. \ref{Fig:phase_trans_ex}. As can be seen in the figure, the phase transition curve for sparse Laplace distribution is slightly higher than that for sparse Gaussian distribution, which means Laplace distribution has less stringent local identifiability conditions. 
\begin{figure}[!hb]
\centering
\includegraphics[scale=.8]{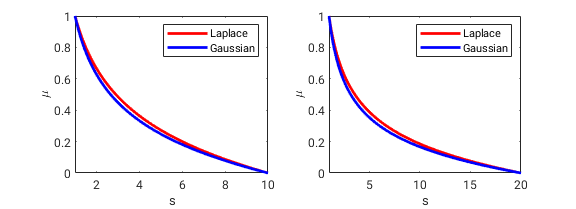}
\caption{The theoretical phase transition curve for constant collinearity dictionary with coherence $\mu$ and sparsity $s$ for $K = 10$ (Left) and $K = 20$ (Right). The phase transition curve corresponds to the asymptotic theoretical boundary that separates the region where local identifiability holds (the area below the curve) and the region where local identifiability fails (the area above the curve) according to Theorem \ref{Thm:local}.}\label{Fig:phase_trans_ex}
\end{figure}
That is consistent with our intuition: while the density function of a Gaussian distribution is rotation symmetric, which implies that it does not "prefer" any direction, the density function of the Laplace distribution is not. For example, let's consider a simple two-dimensional case. Let $\refdict$ be the identity matrix in $\mathbb R^{2\times 2}$. If the reference coefficient is from Gaussian distribution with no sparsity, then all the orthogonal dictionaries will have the same objective value $\sqrt{\frac{2}{\pi}}$. So local identifiability does not hold for Gaussian distribution with no sparsity. However, for Laplace distribution with no sparsity, for an orthogonal dictionary ($\theta \in [0, \pi/2]$)
$
\left(
\begin{array}{cc}
\cos \theta & \sin \theta \\
\sin \theta & - \cos \theta
\end{array}
\right)
$, its $\ell_1$ objective function value would be $2(\sin \theta + \cos \theta + \frac{1}{\sin \theta + \cos \theta})$, which attains its minimum when $\theta = 0$ or $\frac{\pi}{2}$. That means the local identifiability still holds. That shows Laplace distribution should have less stringent conditions for local identifiability .


\newpage
\section*{Appendix B: Additional Simulations}
\label{Sec:Appendix_B}
\subsection*{Running time complexity}
Following the simulation in Section \ref{Sec:5.1}, we carry out the same simulation for different values of $\mu$ and $p$. Let the reference dictionary be a constant collinearity dictionary with coherence $\mu = 0.1$ and $\mu = .9$. The sparse linear coefficients are generated from the Bernoulli Gaussian distribution $BG(p)$ where $p = .1$ and $p = .9$. The simulation results are shown in Fig. \ref{Appendix_Fig:sim1_1} and Fig. \ref{Appendix_Fig:sim1_2}. We find that for a fixed dimension, the computation time scales roughly linearly with sample size, and for fixed sample size, the computation time scales quadratically with dimension $\dim$. That reveals the same trend as the simulation in Section 5.1.
\begin{figure}[ht]
\centering
\includegraphics[scale=.5]{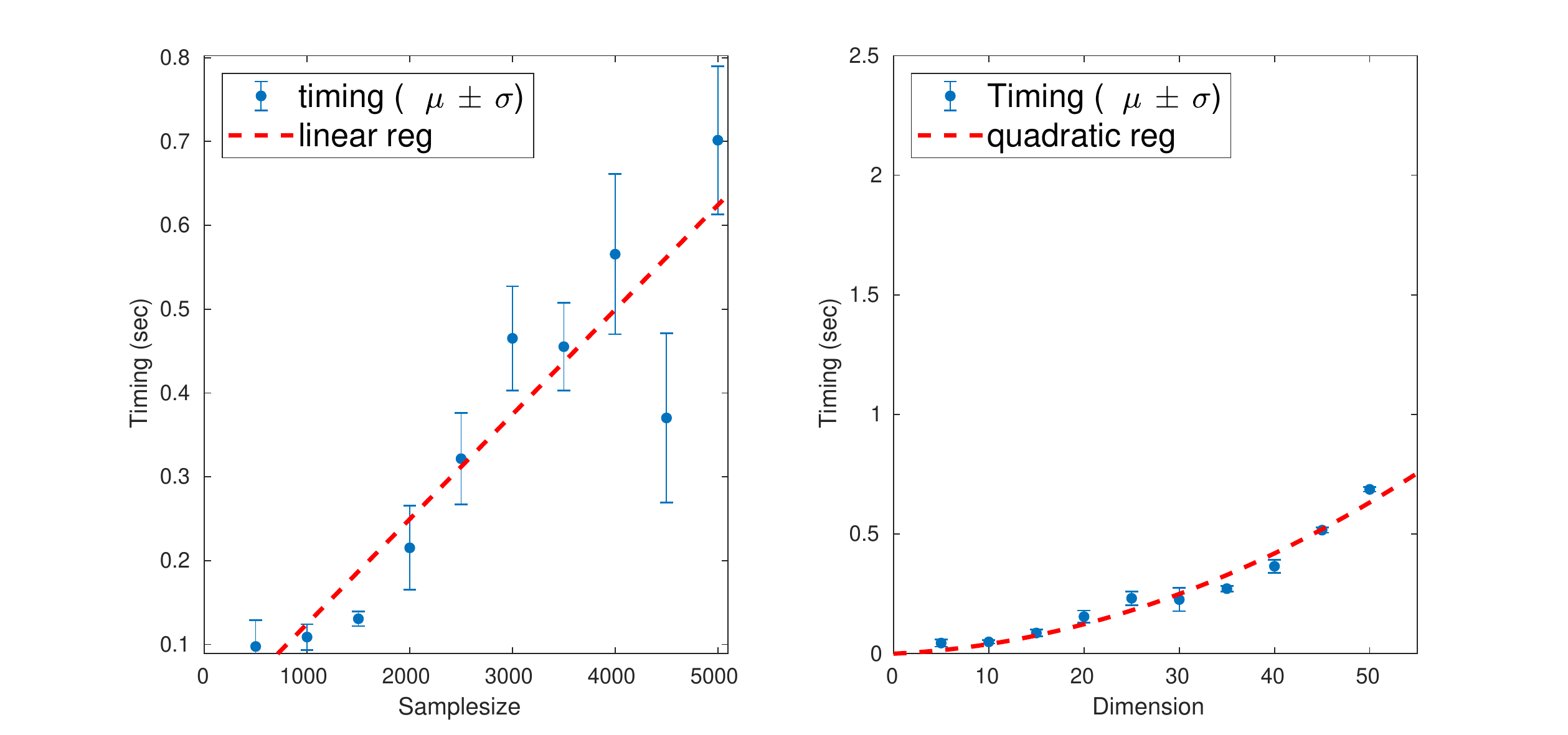}
\caption{Computation time of Algorithm 1. $p = .1$ and $\mu = .1$. (Left) For $\dim = 20$ and $n=500,\ldots, 5000$. (Right) For $\dim = 5,\ldots,50$ and $n = 400$.}\label{Appendix_Fig:sim1_1}
\end{figure}
\begin{figure}[ht]
\centering
\includegraphics[scale=.5]{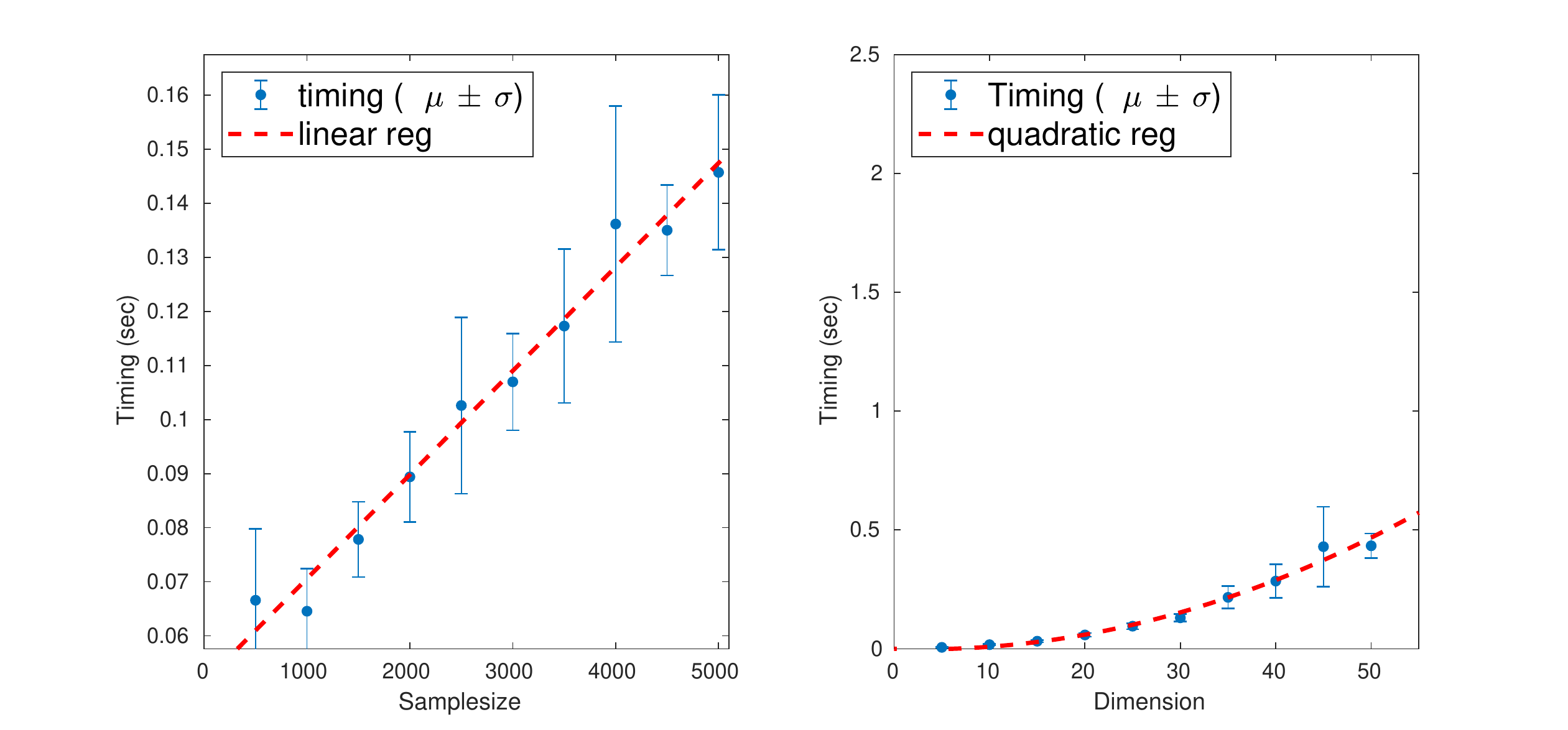}
\caption{Computation time of Algorithm 1. $p = .9$ and $\mu = .9$. (Left) For $\dim = 20$ and $n=500,\ldots, 5000$. (Right) For $\dim = 5,\ldots,50$ and $n = 400$.}\label{Appendix_Fig:sim1_2}
\end{figure}
\subsection*{Sensitivity analysis for $\rho$}
We test the sensitivity of Algorithm \ref{Alg:test} by varying the parameter $\rho$. Let dictionary dimension $\dim = 20$, sparsity parameter $s = 10$ and sample size $n = 1600$. We consider constant collinearity dictionaries with $\mu = \frac{1}{\sqrt{s}}(\frac{\dim - s}{\dim - 1} + 0.05)$ (Fig. \ref{Appendix_Fig:rho_sensitivity} Left) and $\mu = \frac{1}{\sqrt{s}}(\frac{\dim - s}{\dim - 1} - 0.1)$ (Fig. \ref{Appendix_Fig:rho_sensitivity} Right). For the first experiment, the reference dictionary is not a sharp local minimum of the objective function given large enough samples. Hence a small perturbation to Algorithm \ref{Alg:test} will result in a large distance $r$ defined in the algorithm. Similarly, in the second experiment, the reference dictionary is sharp, indicating the distance $r$ in the Algorithm \ref{Alg:test} should be small with respect to perturbation. The results are in Fig. \ref{Appendix_Fig:rho_sensitivity}. This experiment shows for parameter $\rho$ values ranging from $0.05$ to $0.1$, Algorithm \ref{Alg:test} succeeds in distinguishing between the sharp and not-sharp local minima. The smaller their difference is, the smaller $\rho$ we need to use.
\begin{figure}[ht]
\centering
\includegraphics[width=.8\textwidth]{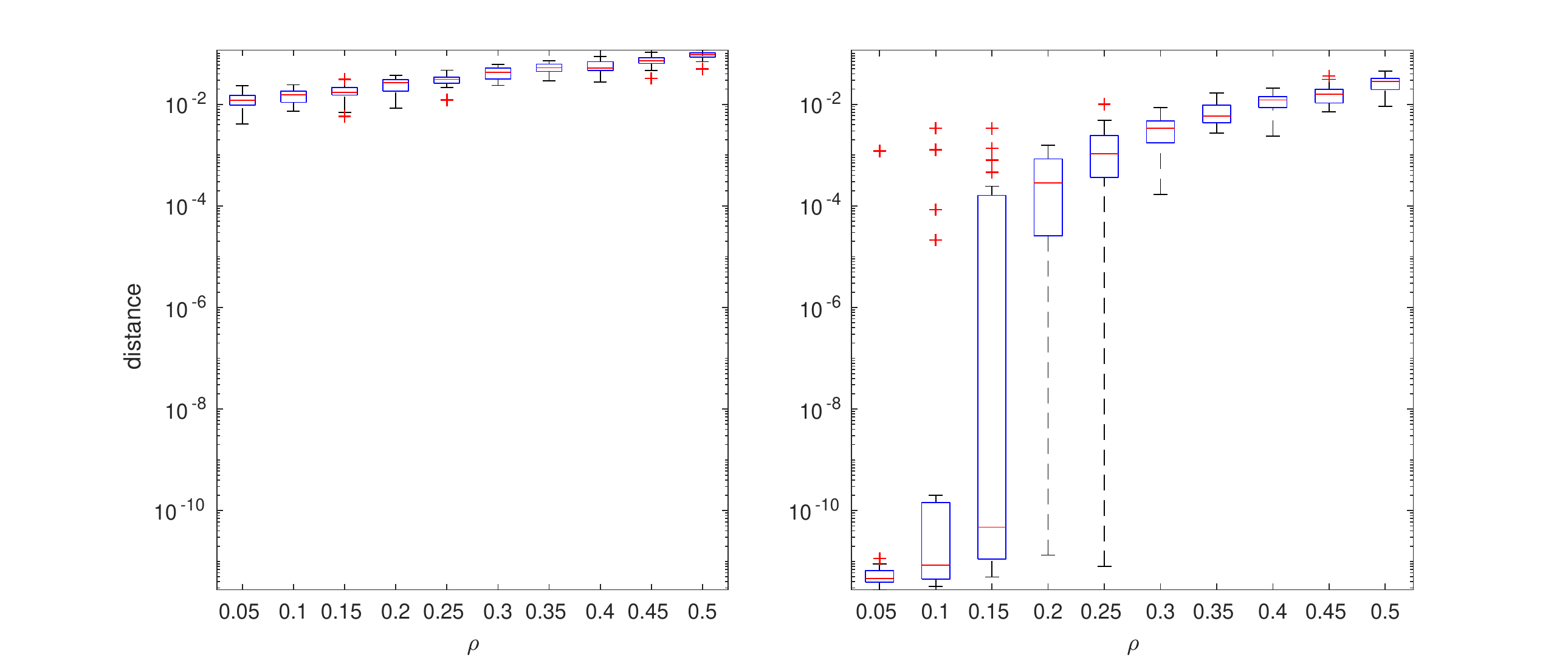}
\caption{Sensitivity analysis of perturbation parameter $\rho$ in Algorithm \ref{Alg:test}. Left: constant collinearity dictionary with $\mu = \frac{1}{\sqrt{s}}(\frac{\dim - s}{\dim - 1} + 0.05)$; Right: constant collinearity dictionary with coherence $\mu = \frac{1}{\sqrt{s}}(\frac{\dim - s}{\dim - 1} - 0.1)$.}\label{Appendix_Fig:rho_sensitivity}
\end{figure}

\newpage
\section*{Appendix C: Proofs}
\label{Sec:Appendix_C}
\subsection*{Proofs of Propositions}
\begin{proof}[Proof of Proposition \ref{prop:specific_constants}]
To prove $\norm{\cdot}_{\coef}$ is lower bounded by $\norm{\cdot}_F$ in the linear subspace $H^\dim = \{A\in \R^{\dim\times\dim}|A_{k,k} = 0\forall~k\}$, we only need to show that $\norm{\cdot}_{\coef}$ is a norm on $H^\dim$. If we can prove it is a norm, then we know it is equivalent to the Frobenious norm since $H^\dim$ is a finite dimensional space.

In order to show that $\norm{\cdot}_{\coef}$ is a norm, we need to prove three properties:
\begin{itemize}
    \item Sub-additivity: for any $A,B\in H^\dim$, $\norm{A + B}_{\coef} \leq \norm{A}_{\coef} + \norm{B}_{\coef} $.
    \item Absolutely homogeneity: for any $A\in H^\dim$ and $\lambda > 0$, $\norm{\lambda A}_{\coef} = \lambda \norm{A}_{\coef}$.
    \item Positive definiteness: If $\norm{A}_{\coef}  = 0$ and $A\in H^\dim$, we know $A = 0$.
\end{itemize}
The first two properties are quite straightforward to show so we omit them in the proof. We focus on proving the third property. Note that $\norm{A}_{\coef}$ is a sum of $\dim$ non-negative terms, if $\norm{A}_{\coef} = 0$, then for any $k \in \{1,\ldots, \dim\}$, each term should be zero, i.e. $\E |\sum_{j}A_{k,j}\coef_j|\1(\coef_k = 0) = 0$. If $\coef$ is from Bernoulli type models $\BG(p_1,\ldots, p_\dim)$, then we could further decompose $\E |\sum_{j}A_{k,j}\coef_j|\1(\coef_k = 0) = 0$ into:
\[
\E |\sum_{j}A_{k,j}\coef_j|\1(\coef_k = 0) = 0 \Leftrightarrow
P(\eta_k = 0) \E |\sum_{j}A_{k,j}\eta_jz_j| = 0 \Leftrightarrow
\E |\sum_{j}A_{k,j}\eta_jz_j| = 0.
\]
The second ``$\Leftrightarrow$" is because $p_k\neq 1$. Since $\E |\sum_{j}A_{k,j}\eta_jz_j| = 0 > P(\eta_1 = \ldots = \eta_\dim = 1)\E |\sum_j A_{k,j}z_j| \geq 0$, since $p_1,\ldots, p_\dim\neq 0$, we know $\E |\sum_j A_{k,j}z_j| = 0$. Define $X = \sum_j A_{k,j}z_j$, since $\E |X| = 0$, we know $X = 0$ almost surely. If $A_{j,k}$ are not all zeros, this means $z_1,\ldots, z_\dim$ are linearly dependent. However, that would contradict the fact that $z$ has density in $\R^\dim$. That completes the proof for Bernoulli type models. For exact sparse models, the approach would be essentially the same.

Now for sparse Gaussian and Bernoulli Gaussian distributions, we could get the specific constant $c_{\coef}$.
We first prove the constant for the sparse Gaussian distribution.
For $X\in H^\dim$, 
\begin{align*}
\norm{X}_{\coef} = &\sqrt{\frac{2}{\pi}}\sum_{k=1}^\dim {\dim \choose s}^{-1}\sum_{ \substack{S\subset \{1,\ldots, \dim\} \\ k\not\in S, |S|=s}}\sqrt{\sum_{j\in S} X_{k,j}^2}\\
= & \frac{s(\dim - s)}{\dim(\dim - 1)}\sqrt{\frac{2}{\pi}}\sum_{k=1}^\dim {\dim - 2 \choose s - 1}^{-1}\sum_{\substack{S\subset \{1,\ldots, \dim\} \\ k \not \in S, |S|=s}}\sqrt{\sum_{j\in S} X_{k,j}^2}\\
\text{(Lemma 6.5 in \cite{Wu2015})} \geq &\frac{s(\dim - s)}{\dim(\dim - 1)}\sqrt{\frac{2}{\pi}} \sum_{k=1}^\dim \sqrt{\sum_{j=1}^\dim X_{k,j}^2}\\
(\|x\|_1\geq \|x\|_2)~\geq & \frac{s(\dim - s)}{\dim(\dim - 1)}\sqrt{\frac{2}{\pi}}\norm{X}_F.
\end{align*}
Here, we need to use Lemma 6.5 in \cite{Wu2015}. For the completeness of this paper, we rewrite that lemma below:
\paragraph{Lemma 6.5 in \cite{Wu2015}} Let $z\in \R^{\dim - 1}$, then for $1\leq l \leq m \leq \dim - 1$,
\[
{\dim - 2 \choose l - 1}^{-1}\sum_{\substack{S\subset \{1,\ldots, \dim-1\}\\ |S|=l}}\sqrt{\sum_{j\in S} z_{j}^2} \geq {\dim - 2 \choose m - 1}^{-1}\sum_{\substack{S\subset \{1,\ldots, \dim-1\}\\ |S|=m}}\sqrt{\sum_{j\in S} z_{j}^2}.
\]
Then the first inequality holds by setting $l=s$ and $m = \dim - 1$.
In summary, we have shown that for $X\in H^\dim$, $\norm{X}_{\coef} \geq \frac{s}{\dim} \sqrt{\frac{2}{\pi}}\norm{X}_F$, which means $c_{\coef}$ is at least $\frac{s(\dim - s)}{\dim(\dim - 1)}\sqrt{\frac{2}{\pi}}$.

Now, let's compute the constant $c_{\coef}$ for Bernoulli Gaussian distribution. For $X\in H^\dim$, if we define $\tilde s = \lceil (\dim - 2)p + 1\rceil$,
\begin{align*}
\norm{X}_{\coef} = &\sqrt{\frac{2}{\pi}}\sum_{k=1}^\dim \sum_{s=0}^{\dim-1} \sum_{\substack{S\subset \{1,\ldots, \dim\} \\ |S|=s, k\not\in S}}p^s(1 - p)^{\dim - s}\sqrt{\sum_{j\in S} X_{k,j}^2}\\
\text{(Lemma 6.6 in \cite{Wu2015})} \geq & (1 - p)\sqrt{\frac{2}{\pi}}\sum_{k=1}^\dim {\dim - 1\choose \tilde s}^{-1}\sum_{\substack{S\subset \{1,\ldots, \dim\} \\ k\not\in S, |S|=\tilde s}}\sqrt{\sum_{j\in S} X_{k,j}^2}\\
\text{(By Example 1)} \geq & (1 - p)\frac{\lceil (\dim - 2)p + 1\rceil}{\dim - 1}\sqrt{\frac{2}{\pi}}\norm{X}_F\geq p(1 - p)\sqrt{\frac{2}{\pi}}\norm{X}_F.
\end{align*}
Here, we used Lemma 6.6 in \cite{Wu2015}. We rewrote that Lemma using the notations in our paper as the following.

\paragraph{Lemma 6.6 in \cite{Wu2015}} Let $p\in (0, 1)$ and $\tilde s = \lceil (\dim - 2)p + 1 \rceil$. For any $z\in \R^{\dim - 1}$, 
\[
\sum_{s=0}^{\dim-1} \sum_{\substack{S\subset \{1,\ldots, \dim - 1\} \\ |S|=s}} p^s(1 - p)^{\dim - 1 - s}\sqrt{\sum_{j\in S} z_{j}^2} \geq {\dim - 1 \choose \tilde s}^{-1} \sum_{\substack{S\subset \{1,\ldots, \dim - 1\}\\ |S|=\tilde s}} \sqrt{\sum_{j\in S} z_{j}^2}.
\]

In summary, we have shown that for $X\in H^\dim$, $\norm{X}_{\coef} \geq p (1-p) \sqrt{\frac{2}{\pi}}\norm{X}_F$, which means $c_{\coef}$ is at least $p(1-p)\sqrt{\frac{2}{\pi}}$.

\end{proof}

\begin{proof}[Proof of Proposition \ref{prop:assumption3_is_general}]
In order to prove Assumption II, we only need to show that for any $c_1,\ldots, c_{\dim}$, $P(\sum_{l=1}^d c_l \coef_l = 0,\text{~and~}\exists~l,~c_l\coef_l \neq 0) = 0$. 
Note that $\coef_j = \xi_j z_j$ for $j=1,\ldots, \dim$,
\begin{align*}
& P(\sum_{l=1}^d c_l \coef_l = 0,\text{~and~}\exists~l,~c_l\coef_l \neq 0)\\
\leq & \sum_{S\subset \{1\ldots, \dim\}}  P(\xi_l = 1~\text{for~}l\in S\text{~and~}0\text{~if~}l\not\in S) \cdot P(\sum_{l\in S} c_l z_l = 0,\text{~and~}\sum_{l\in S} c_l^2 > 0).
\end{align*}
Because $z$ has density, $P(\sum_{l\in S} c_l z_l = 0,\text{~and~}\sum_{l\in S} c_l^2 > 0) = 0$ for any $S$. 
That proves the conclusion.
\end{proof}
\subsection*{Proofs of Corollaries}

\begin{lemma}\label{lemma:symmetry}
If $X$ equals to $c\cdot \mathit{1}\mathit{1}^T$, and $\norm{\cdot}_{\coef} = \sum\limits_{k=1}^\dim \sqrt{\frac{2}{\pi}}\frac{s}{\dim} {\dim - 1 \choose s - 1}^{-1}\sum\limits_{\substack{S\subset \{1,\ldots, \dim\} \\ k \not \in S, |S|=s}}\sqrt{\sum_{j\in S} X_{k,j}^2}$, then $\norm{X}_{\coef}^* = \frac{c\dim(\dim - 1)}{\sqrt{s}(\dim - s)}\sqrt{\frac{\pi}{2}}$. 
\end{lemma}
\begin{proof}[Proof of Lemma \ref{lemma:symmetry}]
Essentially, we are trying to prove that
\[
\max_{A\neq 0, A\in H^\dim} \frac{\tr(A^TX)}{\norm{A}_{\coef}} = \max_{A\neq 0, A\in H^\dim} \frac{c\sum_{k=1}^\dim\sum_{j\neq k}A_{k,j}}{\sum_{k=1}^\dim \frac{s}{\dim} {\dim - 1 \choose s - 1}^{-1}\sum\limits_{\substack{S\subset \{1,\ldots, \dim\} \\ k \not \in S, |S|=s}}\sqrt{\sum_{j\in S} A_{k,j}^2}} \sqrt{\frac{\pi}{2}}= \frac{c\dim(\dim - 1)}{\sqrt{s}(\dim - s)}\sqrt{\frac{\pi}{2}}
\]
Note that this is equivalent to the fact that the following convex optimization problem attains the minimum $(\dim - s)\sqrt{s}$ :
\begin{align*}
\min~& \sum_{k=1}^\dim \frac{s}{\dim} {\dim - 1 \choose s - 1}^{-1}\sum_{\substack{S\subset \{1,\ldots, \dim\} \\ k \not \in S, |S|=s}}\sqrt{\sum_{j\in S} A_{k,j}^2} \\
\textrm{subject~to~} & \sum_{k=1}^\dim\sum_{j\neq k}A_{k,j} = \dim(\dim - 1).
\end{align*}
Note that both the objective and the constraint is permutation symmetric: if $\tilde A$ is obtained by permuting off-diagonal elements from each row in $A$,  then the objective function remains the same. It is not hard to show for the optimal solution $A^*$ must satisfy that for any $k$, $j_1\neq k$, and $j_2\neq k$, $A_{k, j_1}^* = A_{k,j_2}^*$. Therefore, $A_{k,j}^* = 1$ and the objective function is $s {\dim - 1 \choose s - 1}^{-1} {\dim - 1 \choose s} \sqrt{s} = (\dim - s)\sqrt{s}$. That completes the proof.
\end{proof}
\begin{proof}[Proof of Corollary \ref{coro:4}]
 (local identifiability for constant collinearity reference dictionary and sparse Gaussian coefficients) We compute the local identifiability conditions when the reference dictionary is a constant collinearity dictionary with coherence $\mu$ and the coefficients are sparse Gaussian. Specifically, the reference dictionary is 
where $\mathit{1}\mathit{1}^T\in \R^{\dim \times \dim}$ is a square matrix whose elements are all one. The coefficients are generated from sparse Gaussian distribution $SG(s)$. First, the collinearity matrix $M^* = (\refdict)^T\refdict = (1 - \mu)\I + \mu \mathit{1}\mathit{1}^T$. Because $\coef$ is sparse Gaussian, we know $\E \coef_j\sgn (\coef_k) = 0$ for any $j\neq k$. The bias matrix $B$ is
\[
(B(\coef, M^*))_{k, j} = 
\left\{
\begin{array}{ll}
- M_{j,k}\E|\coef_j| = - M_{j,k}\sqrt{\frac{2}{\pi}}\frac{s}{\dim} = - \sqrt{\frac{2}{\pi}}\frac{\mu s}{\dim} & \text{for $j\neq k$}\\
\E|\coef_j| - \E|\coef_j| = 0 & \text{if $j = k$}
\end{array}
\right..
\]
That means $B(\coef, M^*)$ is a constant matrix except for the diagonal elements. 
In general, $\norm{\cdot}_{\coef}^*$ does not have an explicit formula, but for constant matrices, there is a closed form formula (see Lemma \ref{lemma:symmetry}). Using Lemma \ref{lemma:symmetry}, we know
\[
\norm{B(\coef, M^*)}_{\coef}^* = \frac{\mu\sqrt{s}(\dim - 1)}{\dim - s}.
\]
Now let's calculate the sharpness constant and the region bound. Plugging in $\norm{B(\coef, M^*)}_{\coef}^*$, $c_{\coef}$, and $\norm{\refdict}_2^2 = 1 - \mu + \mu(K - 1)$, the sharpness is at least:
\[\frac{1}{\sqrt{\pi}(1 - \mu + \mu (\dim - 1))}\frac{s}{K}\left(1 - \mu \sqrt{s}\frac{\dim - 1}{\dim - s}\right)\approx \frac{s}{\sqrt{\pi}\mu K^2}\left(1 - \mu \sqrt{s}\right)~\text{for large $\dim$}.\]
Because $\max_j \E|\coef_j| = \sqrt{\frac{2}{\pi}}\frac{s}{\dim}$, the region bound in Theorem \ref{Thm:region} is 
\[\norm{\dict - \refdict}_F\leq \frac{1}{8\sqrt{2}(1 - \mu + \mu (\dim - 1))}\left(1 - \mu \sqrt{s}\frac{\dim - 1}{\dim - s}\right)\approx \frac{1}{8\sqrt{2}\mu K}\left(1 - \mu \sqrt{s}\right)~\text{for large $\dim$}.\]
That completes the proof.
\end{proof}

\begin{proof}[Proof of Corollary \ref{coro:5}]
Assume the reference dictionary is a constant collinearity dictionary with coherence $\mu$ and the coefficients are generated from non-negative sparse Gaussian distribution $|SG(s)|$. It can be shown that
\[
(B(\coef, M^*))_{k, j} = 
\left\{
\begin{array}{ll}
- \sqrt{\frac{2}{\pi}}\left(\frac{\mu s}{\dim} - \frac{s(s - 1)}{\dim(\dim - 1)}\right) & \text{for $j\neq k$.}\\
0 & \text{if $j = k$.}
\end{array}
\right.
\]
This shows $B(\coef, M^*)$ is still a constant matrix except the diagonal elements. However, compared with standard sparse Gaussian coefficients, the constant here is $\sqrt{\frac{2}{\pi}}\left(\frac{\mu s}{\dim} - \frac{s(s - 1)}{\dim(\dim - 1)}\right)$, which is smaller than $\sqrt{\frac{2}{\pi}}\frac{\mu s}{\dim}$ in Corollary \ref{coro:4}. Recall the explanation of the matrix $B$ after Theorem \ref{Thm:local}, that is because for non-negative sparse Gaussian coefficients, the "bias" matrix $B_1$ introduced by the coefficient is of different signs compared to the "bias" matrix $B_2$ introduced by the reference dictionary and they cancel with each other. In standard sparse Gaussian case, $B = 0$ if $\mu = 0$, which means the reference dictionary is orthogonal. For this non-negative case, $B = 0$ if $\mu = s/\dim$, which means the atoms in the reference dictionary should have positive collinearity $s/\dim$. As will be shown next, this significantly relaxes the local identifiability condition for non-negative coefficients.

Let's compute the closed form formula for the dual norm. By definition, for any matrix $X$ whose elements are all non-negative, $\norm{X}_{\coef} = \sum_{k=1}^\dim \E |\sum_{j=1}^\dim X_{k,j}\coef_j|\1(\coef_k = 0) = \sum_{k=1}^\dim \sum_{j=1}^\dim X_{k,j}\E\coef_j\1(\coef_k = 0) = \sqrt{\frac{2}{\pi}}\frac{s(\dim - s)}{\dim(\dim - 1)}\sum_{k=1}^\dim \sum_{j=1,j\neq k}^\dim X_{j, k}$. Thus we have
\[
\norm{B(\coef, M^*)}_{\coef}^* = \frac{\sqrt{\frac{2}{\pi}}\frac{s}{\dim}\cdot \Big|\mu - \frac{s-1}{\dim-1}\Big|}{\sqrt{\frac{2}{\pi}}\frac{s(\dim - s)}{\dim(\dim - 1)}} =  \frac{\dim - 1}{\dim - s}\cdot \Big|\mu - \frac{s - 1}{\dim - 1}\Big|.
\]
\end{proof}

\begin{proof}[Proof of Corollary \ref{coro:appendix_1}]
First of all,
\[
(B(\coef, M^*))_{k, j} = 
\left\{
\begin{array}{ll}
 - M_{j,k}\E|\coef_j| = - M_{j,k}\sqrt{\frac{2}{\pi}}p=- \sqrt{\frac{2}{\pi}}\mu p& \text{for $j\neq k$.}\\
0 & \text{if $j = k$.}
\end{array}
\right.
\]
Because all the elements in the matrix are constant except the diagonal elements, similar to Lemma \ref{lemma:symmetry}, we have
\[
\norm{B(\coef, M^*)}_{\coef}^* = \frac{\mu p(\dim - 1)}{(1-p)\sum_{s = 0}^{\dim-1} {\dim-1 \choose s} p^s(1-p)^{\dim - 1 - s}\sqrt{s}}\leq \frac{\mu\sqrt{p(\dim - 1)}}{1-p}.
\]
Here we are using the Jensen inequality that 
$$\sum_{s = 0}^{\dim-1} {\dim-1 \choose s} p^s(1-p)^{\dim - 1 - s}\sqrt{s} > \sqrt{\sum_{s = 0}^{\dim-1} {\dim-1 \choose s} p^s(1-p)^{\dim - 1 - s}s}=\sqrt{(\dim - 1)p}.$$
Thus RHS $< 1$ when $\mu$ and $p$ are small. The sharpness is at least 
\[\frac{p(1-p)}{\sqrt{\pi}(1 - \mu + \mu(\dim - 1))}\left(1-\frac{\mu\sqrt{p(K-1)}}{1-p}\right),\] 
Because $\E |\coef_j| = p\sqrt{\frac{2}{\pi}}$ for any $j$, the bound in Theorem \ref{Thm:region} is 
\[
\frac{1 - p}{8\sqrt{2}(1 - \mu + \mu(\dim - 1))}\left(1-\frac{\mu\sqrt{p(K-1)}}{1-p}\right).
\]
\end{proof}

\begin{proof}[Proof of Corollary \ref{coro:appendix_2}]
We compute the local identifiability condition when the reference dictionary is a constant collinearity dictionary with coherence $\mu$ and the coefficients are generated from sparse Laplace distribution, i.e. for any $j$ $\coef_j = \xi_j z_j$ where $z_j$ is from standard Laplace distribution and $\xi$ is the same as before. Then 
\[
(B(\coef, M^*))_{k, j} = 
\left\{
\begin{array}{ll}
- \mu\frac{s}{\dim}& \text{for $j\neq k$.}\\
0 & \text{if $j = k$.}
\end{array}
\right.
\]
and similar to Lemma \ref{lemma:symmetry}, we have
\[
\norm{B(\coef, M^*)}_{\coef}^* = \frac{\mu s(\dim - 1)}{(\dim - s) \iint_{0}^\infty |y - x|(xy)^{s-1}\exp(-(x+y))\Gamma(s)^{-2}dxdy}.
\]
The reference dictionary is a local minimum when $\norm{B(\coef, M^*)}_{\coef} < 1$.
\end{proof}

\subsection*{Proofs of Theorems}
The following lemmas are useful for proving Theorem \ref{Thm:local}. 

\begin{lemma}\label{lemma:1}
Given two dictionaries $\dict$ and $\dict'\in\BASIS{\R^\dim}$, we have the decomposition:
\[
\inv{\dict}\dict' = \I + (\inv{\dict}\dict' - \I - \Lambda(\dict, \dict')) + \Lambda(\dict, \dict').
\]
where $\Lambda(\dict, \dict')$ is a diagonal matrix whose $j$-th element is $-\frac{1}{2}\|\dict_j - \dict_j'\|_2^2$. Then we know 
\begin{enumerate}
\item For any $j = 1,\ldots, \dim$, $M[j,](\inv{\dict}\dict'_j - \I_j - \Lambda_j(\dict, \dict')) = 0$ where $M = \dict^T\dict$.
\item $\norm{\Lambda(\dict)}_F = \Theta(\norm{\dict - \refdict}_F^2)$:
\[
\frac{1}{2\sqrt{\dim}} \norm{\dict - \refdict}_F^2 \leq \norm{\Lambda(\dict)}_F \leq \frac{1}{2} \norm{\dict - \refdict}_F^2.
\]
\item When $\big<\dict_j, \dict_j'\big>\geq 0$ for any $j = 1,\ldots, \dim$, $\norm{\inv{\dict}\dict' - \I - \Lambda(\dict, \dict')}_F = \Theta(\norm{\dict - \refdict}_F)$:
\[
\frac{\norm{\dict - \dict'}_F}{\sqrt{2} \norm{\dict}_2}\leq  \norm{\inv{\dict}\dict' - \I - \Lambda(\dict, \dict')}_F\leq 
\norm{\inv{\dict}}_{2} \cdot \norm{\dict - \dict'}_F
\]
\item Let $M' = (\dict')^T\dict'$, for any $A$ satisfying $M'[j,] A_j = 0$ for $j=1,\ldots, \dim$ and $\norm{A}_F$ sufficiently small, there is a $\dict\in \BASIS{\R^\dim}$ such that $\inv{\dict}\dict' - \I - \Lambda(\dict, \dict') = A$. 
\end{enumerate}
\end{lemma}

\begin{proof}[Proof of Lemma \ref{lemma:1}]
\textbf{(1)}: 
\begin{align}
& M[j, ](\inv{\dict}\dict'_j - \I_j - \Lambda_j(\dict, \dict'))\\
= & \big<\dict_j, \dict(\inv{\dict}\dict'_j - \I_j - \Lambda_j(\dict, \dict'))\big>\\
= & \big<\dict_j, \dict'_j - \dict_j +\frac{1}{2} \dict_j\|\dict_j - \dict_j'\|^2_2\big>\\
= & \big<\dict_j, \dict'_j - \dict_j\big> + \frac{1}{2}\|\dict_j - \dict_j'\|^2_2\\
= &  \big<\dict_j, \dict'_j\big> - 1 + 1 - \big<\dict_j, \dict'_j\big> = 0.
\end{align}
\textbf{(2)}:
$\norm{\Lambda(\dict, \refdict)}_F = \frac{1}{2} \sqrt{\sum_{j} \|\dict_j - \refdict_j\|_2^4} \leq \frac{1}{2} \norm{\dict - \refdict}_F^2$. On the other hand, $\norm{\Lambda(\dict, \refdict)}_F = \frac{1}{2} \sqrt{\sum_{j} \|\dict_j - \refdict_j\|_2^4} \geq \frac{1}{2\sqrt{\dim}}\norm{\dict - \refdict}_F^2$ because of the power inequality $\|x\|_2 \geq \frac{1}{\sqrt{\dim}} \|x\|_1$.

\textbf{(3)}:
Firstly, consider $\norm{\dict' - \dict - \dict \Lambda(\dict, \dict')}_F$, we have
\[
\norm{\dict' - \dict - \dict \Lambda(\dict, \dict')}_F^2 = \sum_{j=1}^\dim \|\dict_j' - \dict_j\big<\dict_j, \dict_j'\big>\|_2^2 = \sum_{j=1}^\dim 1 - \big<\dict_j, \dict_j'\big>^2 = \sum_{j=1}^\dim \min_{t_j\in \R} \|\dict_j' - t_j\cdot \dict_j\|_2^2.
\]
Then by taking $t_j = 1$ for all $j = 1,\ldots, \dim$, we have 
\[
\norm{\dict' - \dict - \dict \Lambda(\dict, \dict')}_F^2\leq \norm{\dict' - \dict}_F^2.
\]
On the other hand, when $\big<\dict_j, \dict_j'\big> \geq 0$. 
\[
\sum_{j=1}^\dim 1 - \big<\dict_j, \dict_j'\big>^2 \geq \sum_{j=1}^\dim (1 - \big<\dict_j, \dict_j'\big>) ( 1 + \big<\dict_j, \dict_j'\big>) \geq \sum_{j=1}^\dim (1 - \big<\dict_j, \dict_j'\big>)  = \frac{1}{2} \norm{\dict - \dict'}_F^2.
\]
Then for $ (\inv{\dict}\dict' - \I - \Lambda(\dict, \dict')) $, using the above inequalities, we have:

\begin{align}
\norm{\dict' - \dict}_F \leq & \sqrt{2} \norm{\dict' - \dict - \dict \Lambda(\dict, \dict')}_F\\
\leq & \sqrt{2} \norm{\dict}_{2} \norm{\inv{\dict}\dict' - \I - \Lambda(\dict, \dict')}_F,
\end{align}
which proves the first inequality. The second inequality follows from 
\begin{align}
& \norm{\inv{\dict}\dict' - \I - \Lambda(\dict, \dict')}_F\\
\leq & \norm{\inv{\dict}}_{2} \norm{\dict' - \dict - \dict \Lambda(\dict, \dict')}_F\\
\leq &  \norm{\inv{\dict}}_{2} \norm{\dict' - \dict}_F.
\end{align}


\textbf{(4)}: Consider a differentiable mapping $F(\dict) = \inv{\dict}\dict' - \I - \Lambda(\dict, \dict')$ from $\BASIS{\R^{\dim}}$ to a linear manifold
\[
H = \{A\in \R^{\dim\times\dim}\Big| M'[j,]A_j = 0~\text{for~any}~ j = 1,\ldots, \dim.\}
\]
Since $F(\dict') = \vec{0}$, if we can prove the differential of $F$ at $\dict'$, namely   $dF$, is bijective from the tangent space $T \BASIS{\R^{\dim}}\Big|_{\dict'} =  \{A\in \R^{\dim\times\dim}\Big| \big<D'_j, A_j\big> = 0~\text{for~any}~ j = 1,\ldots, \dim.\}$ to the tangent space $T H\Big|_0 = H$, then by the inverse function theorem on the manifold, we will have the conclusion. To prove it is indeed bijective, we note that $dF(\Delta)\Big|_{\dict'}$ is $\sum_{k,j}\inv{(\dict')}_j\I[k,]\Delta_{j,k} = \inv{(\dict')}\Delta$.  Clearly $dF$ is injective: $\inv{(\dict')}\Delta = 0$ implies $\Delta = 0$. To show it is also surjective,  first of all for any $\Delta\in T\BASIS{\R^\dim}\Big|_{\dict'}$, its image under $dF$ is in $H$:
\begin{align*}
M'[j,]\inv{(\dict')}\Delta_j = & \big<\dict_j', \dict'(\inv{(\dict')}\Delta_j\big>\\
= & \big<\dict_j', \Delta_j\big> = 0.
\end{align*}
Because these two linear manifolds have the same dimension, that means $dF$ must be one-on-one.
That completes the proof.
\end{proof}
\begin{lemma}\label{lemma:1+}
If $\norm{\cdot}_{\coef}$ is regular with constant $c_{\coef}$, then we know for any $\dict, \dict'$ such that $\big<\dict_j, \dict_j'\big> \geq 0$ for any $j =1,\ldots, \dim$, $\norm{\inv{(\dict)}\dict'}_{\coef} \geq \frac{c_{\coef}}{\sqrt{2}\norm{\dict}_2^2}\norm{\dict - \dict'}_F$.
\end{lemma}
\begin{proof}
First of all, because for any $A\in \R^{\dim\times\dim}$, by definition of $\norm{\cdot}_{\coef}$, $\norm{A}_{\coef}$ does not depend on diagonal elements $A_{j,j}$ for any $j = 1,\ldots, \dim$. Thus, $\norm{\inv{(\dict)}\dict'}_{\coef} = \norm{\inv{(\dict)}\dict' - \I - \Lambda(\dict, \dict')}_{\coef}$, where $\Lambda$ is defined in Lemma \ref{lemma:1}. If we denote $A$ to be $\inv{(\dict)}\dict' - \I - \Lambda(\dict, \dict')$, then Lemma \ref{lemma:1} shows $M[j, ] A_j = 0$. Since $M_{j, j} = 1$, $A_{j,j} = -M[j, -j]A[-j,j]$. Thus $\|A_j\|_2^2 \leq (M[j, -j]A[-j,j])^2 + \|A[-j, j]\|_2^2 \leq (\|M[j, -j]\|_2^2 + 1)\|A[-j, j]\|_2^2 = \|M[j,]\|_2^2\|A[-j, j]\|_2^2$. Summing over $j$, we have 
\[
\norm{A}_F \leq \max_j \|M[j,]\|_2\sqrt{\sum_j \|A[-j, j]\|_2^2} .
\]
Note that for any $j$, $\|M[j, ]\|_2 = \|\dict_j^T\dict\|_2 \leq \norm{\dict}_2$, thus we have: $\norm{A}_F \leq \norm{\dict}_2 \sqrt{\sum_j \|A[-j, j]\|_2^2} $. On the other hand, by Lemma \ref{lemma:1}, we know $\norm{A}_F \geq \frac{1}{\sqrt{2} \norm{\dict}_2} \norm{\dict - \dict'}_F$. Combining those together, we have 
\begin{align*}
\norm{\inv{(\dict)}\dict'}_{\coef} = & \norm{\inv{(\dict)}\dict' - \I - \Lambda(\dict, \dict')}_{\coef}\\
\geq & c_{\coef}\sqrt{\sum_j \|A[-j, j]\|_2^2} \\
\geq & c_{\coef}\frac{\norm{A}_F}{\norm{\dict}_2}\\
\geq & c_{\coef}\frac{\norm{\dict - \dict'}_F}{\sqrt{2}\norm{\dict}_2^2}.
\end{align*}
\end{proof}
\begin{lemma}\label{lemma:2}
for $x, y\in \R$, $y\cdot \sgn(x) + |x| \leq |y + x| \leq y\cdot \sgn(x) + |x| + 2|y|\cdot \1(|y| > |x|)$.
\end{lemma}
\begin{proof}
When $|y| < |x|$, $\sgn(x+y) = \sgn(x)$, so $|y + x| = \sgn(x)(x + y) = |x| +  \sgn(x) y$.
When $|y| > |x|$, $\sgn(x + y) = \sgn(y)$, so $|y + x| = |y| + \sgn(y) x \leq |x| + 2|y| + y\sgn(x)$. That completes the proof.
\end{proof}
\begin{lemma}\label{lemma:3}
we have the upper and lower bound of the objective function:
\begin{align*}
&\E\|\inv{(\refdict)}\signal\|_1+
\norm{\inv{\dict}\refdict}_{\coef} - \tr(B(\coef, M)^T\inv{\dict}\refdict) + o(\norm{\dict - \refdict}_F) \\
&\geq \E\|\inv{\dict}\signal\|_1  \geq \\
& \E\|\inv{(\refdict)}\signal\|_1
+ \norm{\inv{\dict}\refdict}_{\coef} - \tr(B(\coef, M)^T\inv{\dict}\refdict) - \E \|\Lambda \coef\|_1
\end{align*}
\end{lemma}
\begin{proof}[Proof of Lemma \ref{lemma:3}]
By Lemma \ref{lemma:1},
$\inv{(\dict)}\refdict$ can be decomposed into
\begin{align*}
\inv{\dict} \refdict = \I + \Delta(\dict, \refdict) + \Lambda(\dict, \refdict),
\end{align*}
where $\Delta(\dict, \refdict) = \inv{\dict}\refdict - \I - \Lambda(\dict, \refdict)$ and $\Lambda(\dict, \refdict)$ is defined in Lemma \ref{lemma:1}. In the following, we will use $\Lambda, \Delta$ without writing $\dict, \refdict$ explicitly.

Let $\Delta_{k, j}$ be the element of $\Delta$ at $k$-th row and $j$-th column. 

Then the objective function can be lower bounded:

\begin{align}
\E\|\inv{\dict}\signal\|_1  = & \E\|\inv{(\refdict)}\signal - (\I - \inv{\dict}\refdict)\inv{(\refdict)}\signal\|_1\\
= & \E\|\coef + (\Delta + \Lambda)\coef\|_1\\
(a) \geq & \E\|\coef + \Delta\coef \|_1 - \E \|\Lambda \coef\|_1\\
(b)  \geq & \E\sum_{k=1}^\dim |\coef_k| + \1(\coef_k = 0) |\sum_j \Delta_{k, j}\coef_j| - \sgn \coef_k \sum_j \Delta_{k, j} \coef_j - \E \|\Lambda \coef\|_1\label{Eq:split}\\
\geq & \E \|\coef\|_1 + \norm{\Delta}_{\coef} - \E \sum_{k,j} \Delta_{k, j} \E \coef_j\sgn \coef_k - \E \|\Lambda \coef\|_1.
\end{align}
(a) holds because of triangle inequality.
(b) holds because of Lemma \ref{lemma:2}. Note that by the definition of $\norm{\cdot}_{\coef}$, the diagonal elements of $\Delta$ do not matter, so $\norm{\Delta}_{\coef} = \norm{\inv{\dict}\refdict}_{\coef}$.

Recall $M_{j,k} = \big<\dict_j, \dict_k\big>$, by Lemma \ref{lemma:1}, $\Delta_{k, j}$ satisfies: $M[j,]\Delta_{j} = \sum_{k\neq j} M_{j, k} \Delta_{k, j} + \Delta_{j, j}  = 0$ (Because $M_{j, j} = 1$) for any $j$. Thus we have 
\begin{align}
\sum_{j,k=1}^\dim \Delta_{k, j} \E\coef_j\sgn \coef_k= &
\sum_{j=1}^\dim \left(\sum_{k\neq j} \Delta_{k, j} \E\coef_j\sgn \coef_k + \Delta_{j,j} \E|\coef_j|\right)\\
=&\sum_{j=1}^\dim \sum_{k\neq j} \Delta_{k, j} \left(\E\coef_j\sgn \coef_k - M_{j, k}\E|\coef_j|\right) = \tr(B(\coef, M)^T\Delta).
\end{align}
Because the diagonal elements of $B(\coef, M)$ are all zeros, we know $\tr(B(\coef, M)^T\Delta) = \tr(B(\coef, M)^T\inv{\dict}\refdict)$. 

In summary, we have shown that 
\begin{align}
\E\|\inv{\dict}\signal\|_1  \geq & \E \|\coef\|_1 + \norm{\inv{\dict}\refdict}_{\coef} - \tr(B(\coef, M)^T\inv{\dict}\refdict) - \E \|\Lambda \coef\|_1.
\end{align}

In order to have an upper bound, we have 
\begin{align}
\E\|\inv{\dict}\signal\|_1  = & \E\|\inv{(\refdict)}\signal - (\I - \inv{\dict}\refdict)\inv{(\refdict)}\signal\|_1\\
= & \E\|\coef + (\Delta + \Lambda)\coef\|_1\\
\leq & \E\|\coef + \Delta\coef \|_1 + \E \|\Lambda \coef\|_1\\
(\text{Lemma \ref{lemma:2}})\leq & \E\|\coef\|_1 + \norm{\inv{\dict}\refdict}_{\coef} - \tr(B(\coef, M)^T\inv{\dict}\refdict)\\
&\quad + \sum_k 2\E \left|\sum_j\Delta_{k, j}\coef_j\right|\1\left(\left|\sum_j\Delta_{k, j}\coef_j\right|>\left|\coef_k\right|\right) + \E \|\Lambda \coef\|_1.
\end{align}
Note that by Lemma \ref{lemma:1},  $\E \|\Lambda \coef\|_1 \leq \norm{\dict - \refdict}_F^2\max_j \E |\coef_j| = o(\norm{\dict - \refdict}_F)$. Furthermore,
\begin{align}
&\E \left|\sum_j\Delta_{k, j}\coef_j\right|\1\left(\left|\sum_j\Delta_{k, j}\coef_j\right|>\left|\coef_k\right|\right)\label{Eq:last1}\\
\leq & \sum_{k=1}^\dim \max_{j}|\Delta_{k,j}| \cdot \E ~\1(\coef_k \neq 0) \1(|\sum_j\Delta_{k,j}\coef_j|\geq |\coef_k|) \|\coef\|_1.
\end{align}
Because $\1(\coef_k \neq 0) \1(|\sum_j\Delta_{k,j}\coef_j|\geq |\coef_k|) \|\coef\|_1\leq \|\coef\|_1$, $\E \|\coef\|_1 < \infty$, and 
\[
\lim_{\Delta_{k,j}\rightarrow 0}  \1(\coef_k \neq 0) \1(|\sum_j\Delta_{k,j}\coef_j|\geq |\coef_k|) \|\coef\|_1 = 0 \quad a.s.,
\]
by the dominant convergence theorem, we know 
\[
\lim_{\Delta \rightarrow 0} 
\E \1(\coef_k \neq 0) \1(|\sum_j\Delta_{k,j}\coef_j|\geq |\coef_k|) \|\coef\|_1
= 
\E \lim_{\Delta\rightarrow 0}  \1(\coef_k \neq 0) \1(|\sum_j\Delta_{k,j}\coef_j|\geq |\coef_k|) \|\coef\|_1 
=0.
\]
This means \eqref{Eq:last1} is $o(\norm{\dict - \refdict}_F)$, which proves the upper bound.
\end{proof}

\begin{proof}[Proof of Theorem 1]

\textbf{(i)}: Let's first prove that if $\norm{\cdot}_{\coef}$ is regular with constant $c_{\coef}$ and \eqref{Eq:local_sharp_condition} holds, $\refdict$ is a sharp local minimum.
When \eqref{Eq:local_sharp_condition} is satisfied and $\dict \to \refdict$, $\norm{B(\coef, M)}_{\coef}^* \rightarrow \norm{B(\coef, M^*)}_{\coef}^* < 1$ and $\norm{\inv{\dict}\refdict}_{\coef} - \tr(B(\coef, M)^T\inv{\dict}\refdict)= \norm{\inv{\dict}\refdict}_{\coef} - \tr(B(\coef, M^*)^T\inv{\dict}\refdict) + o(\norm{\inv{\dict}\refdict}_{\coef}) \geq (1 - \norm{B(\coef, M)}_{\coef}^*)\norm{\inv{\dict}\refdict}_{\coef} + o(\norm{\inv{\dict}\refdict}_{\coef})$. Because $\norm{\cdot}_{\coef}$ is regular and Lemma \ref{lemma:1+}, by appropriately choosing signs of each column in $\refdict$, we have
\begin{align*}
\norm{\inv{\dict}\refdict}_{\coef} \geq \frac{c_{\coef}}{\sqrt{2}\norm{\dict}_2^2}\norm{\refdict - \dict}_{F}.
\end{align*}

Combine those two inequalities, when $\norm{\dict - \refdict}_F$ is small enough, 
\[
\E\|\inv{\dict}\signal\|_1 - \E\|\coef\|_1\geq 
(1 - \norm{B(\coef, M^*)}_{\coef}^*)\frac{c_{\coef}}{\sqrt{2}\cdot \norm{\refdict}_2^2}\norm{\dict - \refdict}_{F}
+ o(\norm{\dict - \refdict}_F).
\]

By Definition \ref{Def:sharplocalmin}, $\refdict$ is a `sharp' local minimum with sharpness at least $(1 - \norm{B(\coef, M^*)}_{\coef}^*)\frac{c_{\coef}}{\sqrt{2} \norm{\refdict}_2^2}$. 

\textbf{(ii)} When \eqref{Eq:local_sharp_condition} does not hold or $\norm{\cdot}_{\coef}$ is not regular, $\refdict$ is not a sharp local minimum.

If $\norm{B(\coef, M^*)}_{\coef}^* \geq 1$, then there exists $\Delta$ such that $M^*[j,] \Delta_j = 0$ for any $j$ and $\norm{\Delta}_{\coef} - \tr(B(\coef, M^*)^T\Delta) \leq 0$, then for any $t > 0$, by Lemma \ref{lemma:1} we could construct a series of dictionaries $\dict(t)$ for a sufficiently small $t$ such that
\[
\inv{(\dict(t))}\refdict = \I + t(\norm{\Delta}_{\coef} - \tr(B(\coef, M^*)^T\Delta)) + o(\|\dict(t) - \refdict\|_F).
\]
Then by Lemma \ref{lemma:3}, we have the formula for the objective of $\dict(t)$:
\begin{align}
\E\|\inv{\dict(t)}\signal\|_1  = & \E\|\coef\|_1 + t\cdot \left(\norm{\Delta}_{\coef} - \tr(B(\coef, M^*)^T\Delta)\right) +  o(\|\dict(t) - \refdict\|_F).
\end{align}

Because $\norm{\Delta}_{\coef} - \tr(B(\coef, M^*)^T\Delta) \leq 0$, $\E\|\inv{\dict(t)}\signal\|_1 \leq \E\|\coef\|_1 + o(\|\dict(t) - \refdict\|_F)$. By definition, $\refdict$ is not a sharp local minimum. If $\norm{\cdot}_{\coef}$ is not regular, for any $c > 0$, there exists $\Delta$ such that $M^*[j,] \Delta_j = 0$ for any $j$ and $\norm{\Delta}_{\coef} < c \|\Delta\|_F$. Without loss of generality, assume $\tr(B(\coef, M^*)^T\Delta)\geq 0$, otherwise just take $-\Delta$. For sufficiently small $t$, there exists a dictionary $\dict(t)$ such that 
\[
\inv{(\dict(t))}\refdict = \I + t(\norm{\Delta}_{\coef} - \tr(B(\coef, M^*)^T\Delta)) + o(\|\dict(t) - \refdict\|_F).
\]
Then by Lemma \ref{lemma:3}, we have the formula for the objective of $\dict(t)$:
\begin{align}
\E\|\inv{\dict(t)}\signal\|_1  = & \E\|\coef\|_1 + t\cdot \left(\norm{\Delta}_{\coef} - \tr(B(\coef, M^*)^T\Delta)\right) +  o(t)\leq \E\|\coef\|_1 + ct + o(t).
\end{align}
Because that holds for any $c > 0$, by definition, we have shown $\refdict$ is not a `sharp' local minimum.

\textbf{(iii)}: When $\norm{B(\coef, M^*)}_{\coef}^* > 1$, $\refdict$ is not a local minimum. This part is essentially the same as (ii). The key is to construct a series of dictionaries $\dict(t)$ using Lemma \ref{lemma:1} as in (ii). Then by using the upper bound in Lemma \ref{lemma:3}, we can find a small $t>0$ and a small $c > 0$ such that
\[
\E\|\inv{\dict_t}\signal\|_1 \leq \E\|\coef\|_1 - ct + o(t).
\]
Thus by definition $\refdict$ is not a local minimum.
\end{proof}
\begin{proof}[Proof of Theorem \ref{Thm:region}]
Note that by Lemma \ref{lemma:1}, we have
\begin{align}
\E \|\Lambda(\dict, \refdict) \coef\|_1 \leq & \max_j 
\E |\coef_j| \norm{\dict - \refdict}_F^2
\end{align}
On the other hand, by Lemma \ref{lemma:3}, we know
\begin{align*}
\E \|\inv{\dict}\signal\|_1 - \E \|\coef\|_1 \geq& \norm{\inv{\dict}\refdict}_{\coef} - \tr(B(\coef, M)^T\inv{\dict}\refdict) - \E \|\Lambda(\dict, \refdict) \coef\|_1
\end{align*}
Similar to the proof of Theorem 1, the right hand side is bounded by
\begin{align}
&\norm{\inv{\dict}\refdict}_{\coef} - \tr(B(\coef, M)^T\inv{\dict}\refdict) - \E \|\Lambda(\dict, \refdict) \coef\|_1\nonumber\\
\geq & (1 - \norm{B(\coef, M)}_{\coef}^*) \norm{\inv{\dict}\refdict}_{\coef} - \E \|\Lambda(\dict, \refdict) \coef\|_1\nonumber\\
\geq & (1 - \norm{B(\coef, M)}_{\coef}^*) \norm{\inv{\dict}\refdict}_{\coef} - \max_j \E |\coef_j| \cdot \norm{\dict - \refdict}_F^2\nonumber\\
\geq & (1 - \norm{B(\coef, M)}_{\coef}^*)\frac{c_{\coef}}{\sqrt{2}\norm{\dict}_2^2} \norm{\dict - \refdict}_F - \max_j \E |\coef_j| \cdot \norm{\dict - \refdict}_F^2\\
\geq & (1 - \norm{B(\coef, M)}_{\coef}^*)\frac{c_{\coef}}{4\sqrt{2}\norm{\refdict}_2^2} \norm{\dict - \refdict}_F - \max_j \E |\coef_j| \cdot \norm{\dict - \refdict}_F^2.\label{Eq:to_continue}
\end{align}

Because $\norm{M - M^*}_F \leq (\norm{\dict}_2 + \norm{\refdict}_2)\cdot \norm{\dict - \refdict}_F\leq 3\norm{\refdict}_2\cdot \norm{\dict - \refdict}_F$ and $\norm{\dict - \refdict}_F\leq \frac{c_{\coef}(1 - \norm{B(\coef, M^*)}_{\coef}^*) }{8\sqrt{2} \max_j \E |\coef_j|\norm{\refdict}_2}$
we know $\norm{M - M^*}_F \leq \frac{c_{\coef}(1 - \norm{B(\coef, M^*)}_{\coef}^*)}{2\max_j \E|\coef_j|\norm{\refdict}_2}\leq \frac{c_{\coef}(1 - \norm{B(\coef, M^*)}_{\coef}^*)}{2\max_j \E|\coef_j|}$. Based on this inequality, we have:
\begin{align*}
1 - \norm{B(\coef, M)}_{\coef}^* \geq & 1 - \norm{B(\coef, M^*)}_{\coef}^* - |\norm{B(\coef, M)}_{\coef}^* - \norm{B(\coef, M^*)}_{\coef}^*|\\
\geq & 1 - \norm{B(\coef, M^*)}_{\coef}^* - \norm{B(\coef, M) - B(\coef, M^*)}_{\coef}^*\\
\geq & 1 - \norm{B(\coef, M^*)}_{\coef}^* - \frac{1}{c_{\coef}}\norm{B(\coef, M) - B(\coef, M^*)}_F\\
\geq & 1 - \norm{B(\coef, M^*)}_{\coef}^* - \frac{1}{c_{\coef}}\max_j \E |\coef_j|\cdot \norm{M - M^*}_F\\
\geq & \frac{1}{2}(1 - \norm{B(\coef, M^*)}_{\coef}^*).
\end{align*}
Based on this, \eqref{Eq:to_continue} is bounded by:
\begin{align}
&\frac{1}{2}(1 - \norm{B(\coef, M^*)}_{\coef}^*)\frac{c_{\coef}}{4\sqrt{2}\norm{\refdict}_2^2} \norm{\dict - \refdict}_F - \max_j \E |\coef_j| \cdot \norm{\dict - \refdict}_F^2\\
\geq & \left(\frac{c_{\coef}(1 - \norm{B(\coef, M^*)}_{\coef}^*) }{8\sqrt{2} \max_j \E |\coef_j|\norm{\refdict}_2^2}- \norm{\dict - \refdict}_F\right)  \norm{\dict - \refdict}_F \max_j \E |\coef_j|\geq 0.
\end{align}
This shows the LHS is positive when $\dict \neq \refdict$ and we have completed the proof.
\end{proof}
\begin{proof}[Proof of Theorem \ref{Thm:unique}]
In order to prove Theorem \ref{Thm:unique}, it suffices to prove any dictionary $\dict$ in $\BASIS{\R^\dim}$ other than $\refdict$ will not be a `sharp' local minimum. Recall $\bm \beta(\dict)$ is the coefficient of the samples under dictionary $\dict$, i.e., $\bm \beta(\dict) = \inv{\dict} \signal$. For notation ease, we will omit $\dict$ and simply write $\bm \beta$. 

The following lemma provides a necessary condition for a dictionary to be a `sharp' local minimum. 
\begin{lemma}\label{lemma:full_rank}
For any dictionary $\dict$, if $\dict$ is a `sharp' local minimum of optimization form \eqref{Eq:main2}, then for any $k=1,\ldots, \dim$, $\bm \beta \cdot \1(\bm \beta_k = 0)$ does not lie in any linear subspace of dimension $\dim - 2$. 
\end{lemma}
\begin{proof}[Proof of Lemma \ref{lemma:full_rank}]
If $\dict$ is a sharp local minimum, by the proof of Theorem 1, it should satisfy \eqref{Eq:local_sharp_condition_general}.

\begin{equation}\label{Eq:local_sharp_condition_general}
\sum_{j, k} \Delta_{k, j} (\E \bm \beta_j \sgn(\bm \beta_k) - M_{j, k} \E |\bm \beta_j|) < \sum_{k}\E |\sum_{j} \Delta_{k, j} \bm \beta_j|\1(\bm \beta_k = 0).
\end{equation}
For any $\Delta_{k, j}$, let $\Delta'_{k, j} \triangleq -\Delta_{k, j}$, it should also satisfy \eqref{Eq:local_sharp_condition_general}. That makes
\[
-\sum_{j, k} \Delta_{k, j} (\E \bm \beta_j \sgn(\bm \beta_k) - M_{j, k} \E |\bm \beta_j|) < \sum_{k}\E |\sum_{j} \Delta_{k, j} \bm \beta_j|\1(\bm \beta_k = 0).
\]
Thus we have 
\begin{equation}\label{Eq:strict_positive}
\E |\sum_{j=1,j\neq k}^\dim \Delta_{k, j} \bm \beta_j|\1(\bm \beta_k = 0) > 0.
\end{equation}
If $\bm \beta \1(\beta_k = 0)$ lies in a linear subspace of dimension $\dim - 2$, because there are $\dim - 1$ free parameters in $\Delta_{j,k}$ for $j \neq k$, we can find a set of nonzero $\Delta_{j,k}$ such that $\sum_{j=1,j\neq k}^\dim \Delta_{k, j} \bm \beta_j\cdot \1(\bm \beta_k = 0) = 0$. That contradicts \eqref{Eq:strict_positive}. Therefore, $\bm \beta \1(\beta_k = 0)$ does not lie in any linear subspace of dimension $\dim - 2$.
\end{proof}

In order to show $\dict\neq \refdict$ up to sign-permutation is not a `sharp' local minimum, by Lemma \ref{lemma:full_rank}, it suffices to find a $k$ such that almost surely the random vector $\bm \beta \cdot \1(\bm \beta_k = 0)$ lies in a linear manifold of dimension at most $\dim - 2$. 

Note that $\bm \beta = \inv{\dict}\refdict \coef$ is linear transform of $\coef$. For $\dict\neq \refdict$ up to the sign-permutation sense, $\inv{\dict}\refdict\neq \I$, which means there exists $k$ such that $\bm \beta_k \neq \coef_{k'}$ for any $k' = 1,\ldots, \dim$. This means $\bm \beta_k$ is the linear combination of at least two elements in $\coef$. Without loss of generality, $\bm \beta_k = \sum_{l=1}^T c_l \coef_l$ such that $c_1,\ldots, c_T\neq 0$ and $T \geq 2$. Because of Assumption I, $\bm \beta_k = 0$ implies $\coef_1 = \ldots = \coef_T = 0$. Thus,
$\bm \beta \1(\beta_k = 0) = \inv{\dict}\refdict \coef \1(\coef_1=\ldots=\coef_T = 0)$, we know $\bm \beta \1(\beta_1 = 0)$ lies in a linear manifold of dimension $\dim - T$ almost surely. That completes the proof.
\end{proof}

\begin{proof}[Proof of Theorem \ref{Thm:global}]
The most important step is to show that as $\frac{\dim \ln \dim}{n} \rightarrow 0$, the finite population satisfies the assumption I asymptotically:
\begin{align}
\sup_{c_1,\ldots, c_\dim}\frac{1}{n}\sum_{i=1}^n \1(\sum_{j} c_j \coef_j^{(i)} = 0~\text{and}~ \sum_{j=1}^\dim \left(c_j\coef_j^{(i)}\right)^2 > 0) \rightarrow 0\quad \text{a.s..}\label{Eq:finite_sample_essense}
\end{align}
Then following essentially the same steps in the proof of Theorem 3, it is easy to show that the sharpness of any local minimum other than the reference dictionary will go to zero. 

In order to prove \eqref{Eq:finite_sample_essense}, let's define $f_c(\coef) \triangleq \1(\sum_{j=1}^\dim c_j \coef_j = 0\text{~and~}\sum_{j=1}^\dim \left(c_j\coef_j^{(i)}\right)^2 > 0)$, $\mathcal{F}(\coef) \triangleq \{f_c(\cdot)|c\in \R^{\dim}\} $ and consider its VC dimension. We are going to prove the VC dimension of $\mathcal{F}$ is no bigger than $2\dim$, namely, for any $\coef^{(1)}, \ldots, \coef^{(2\dim)}$, define a set
\[
\mathcal{F}^{(2\dim)}(\coef^{(1)},\ldots, \coef^{(2\dim)}) \triangleq \{(f_c(\coef^{(1)}),\ldots, f_c(\coef^{(2\dim)}))|c\in \R^\dim\},
\]
The cardinality of $\mathcal{F}^{(2\dim)}$ is not $2^{2\dim}$.
If $\underbrace{(1,\ldots, 1)}_\text{$2\dim$}$ is not in $\mathcal{F}^{(2\dim)}$, then we are done. Otherwise, there exists $c$ s.t. $f_c(\coef^{(i)}) = 1$ for any $i = 1,\ldots, 2\dim$. That means $\sum_j c_j \coef_j^{(i)} = 0$ for any $i = 1,\ldots, 2\dim$. Therefore, the dimension of the linear space spanned by $\coef^{(1)},\ldots, \coef^{(2\dim)}$ is at most $\dim - 1$. So we could find $\dim - 1$ coefficients such that all other coefficients are their linear combinations. Without loss of generality, let's assume those coefficients are $\coef^{(1)}, \ldots, \coef^{(\dim - 1)}$. Define the support of a vector to be the entries where it is nonzero. For $\coef^{(\dim)},\ldots, \coef^{(2\dim)}$, there will be one coefficient whose support is contained in the union of all the other coefficients. If that's not the case, each coefficient can be mapped to one entry which is only contained in its own support but not any support of other coefficients. But there are $\dim + 1$ coefficient and only $\dim$ entries, which will make a contradiction. Without loss of generality, let's assume 
that coefficient is $\coef^{(\dim)}$. Then let's show that $(\underbrace{1,\ldots, 1}_{\dim},0,\ldots, 0)\not\in \mathcal{F}^{(2\dim)}(\coef^{(1)},\ldots, \coef^{(2\dim)})$. Since $f_c(\coef^{(i)}) = 1$ for $i = 1,\ldots, \dim - 1$, we have 
\[
\sum_j c_j \coef_j^{(i)} = 0\quad \forall~i=1,\ldots, \dim - 1.
\]
Because $\coef^{(\dim)},\ldots, \coef^{(2\dim)}$ are linear combinations of $\coef^{(1)},\ldots, \coef^{(\dim - 1)}$, we know 
\[
\sum_j c_j \coef_j^{(i)} = 0\quad \forall~i=\dim ,\ldots, 2\dim.
\]
If $f_c(\coef^{(i)}) = 0$ for $i = \dim + 1, \ldots, 2\dim$, it means 
\[
\sum_j \left(c_j \coef_j^{(i)}\right)^2 = 0\quad \forall~i = \dim + 1, \ldots, 2\dim,
\]
which means the support of $c$ does not overlap with the support of $\coef^{(\dim+1)},\ldots, \coef^{(2\dim)}$. However, the support of $\coef^{(\dim)}$ is contained in the union of the supports of $\coef^{(\dim+1)},\ldots, \coef^{(2\dim)}$. That means $f_c(\coef^{(\dim)} = 0$ not 1, a contradictory. 

Then by classic results in statistical learning, for example, see \cite{Martin2017}, we know $\mathcal{F}$ is Glivenko–Cantelli and \eqref{Eq:finite_sample_essense} holds as long as $\frac{\dim \ln \dim}{n} \rightarrow 0$.

If we can prove:
\begin{align*}
& P(\text{there exists a local minimum $\dict\neq \refdict$ with sharpness at least $\epsilon$  and $\norm{\dict^{-1}}_2\leq \rho$}) \\
\leq &  P(\sup_{c_1,\ldots, c_\dim}\frac{1}{n}\sum_{i=1}^n f_c(\coef^{(i)}) > \epsilon') .
\end{align*}
Then, using \eqref{Eq:finite_sample_essense}, we will get the desired conclusion.

Now we are going to prove that for any $\epsilon, \rho > 0$, there exists $\epsilon' > 0$ and if $\dict\neq \refdict$ is a local min with sharpness at least $\epsilon$ and $\norm{\dict^{-1}}_2 \leq \rho$, then $\sup_{c_1,\ldots, c_\dim}\frac{1}{n}\sum_{i=1}^n f_c(\coef^{(i)})> \epsilon'.$ If $\dict$ is a local minimum with sharpness at least $\epsilon$, by definition, we have
\[
\frac{1}{n}\sum_{i=1}^n \|\inv{(\dict')}\signal^{(i)}\|_1 - \frac{1}{n}\sum_{i=1}^n \|\inv{\dict}\signal^{(i)}\|_1 \geq \epsilon \norm{\dict' - \dict}_F + o(\norm{\dict' - \dict}_F).
\]
Denote $\bm \beta^{(i)} = \inv{\dict}\signal^{(i)}$ for $i = 1,\ldots, n$, by Lemma \ref{lemma:1}, we have 
\[
\norm{\inv{(\dict')}\dict}_{\bm \beta} - \tr((\inv{(\dict')}\dict)^TB(\bm \beta, \dict^T\dict) \geq \epsilon \norm{\dict' - \dict}_F + o(\norm{\dict' - \dict}_F).
\]
Without loss of generality, we could select $\dict'$ (or $-\dict'$) such that $\tr((\inv{(\dict')}\dict)^TB(\beta, \dict^T\dict) \geq 0$, this leads to
\[
\norm{\inv{(\dict')}\dict}_\beta \geq \epsilon \norm{\dict' - \dict}_F + o(\norm{\dict' - \dict}_F).
\]
Now let's find a $w\neq 0$, $\|w\|_2 = 1$, and a dictionary $\dict'\neq \dict$ such that $\norm{\inv{(\dict')}\dict}_\beta = \sum_{k}\frac{1}{n}\sum_i |\sum_j w_j \beta^{(i)}_j|\1(\beta_k^{(i)} = 0) > \epsilon/2 \norm{\dict' - \dict}_F > 0$. Since $\dict \neq \refdict$ up to sign-permutation ambiguity, that means at least one row of $\inv{\dict}\refdict$ contains two nonzero elements. Without loss of generality, assume the $k$-th row of $\inv{\dict}\refdict$, denoted as $c^{(k)}$, has at least two nonzero entries. We are going to prove it satisfies the desired condition: 
\[
\frac{1}{n}\sum_{i=1}^n f_{c^{(k)}}(\coef^{(i)}) > \epsilon'.
\]
Intuitively we would like to show there are many $\coef^{(i)}$'s  that satisfy $f_{c^{(k)}}(\coef^{(i)}) = 1$, it means $\sum_{j} c_j^{(k)} \coef_j^{(i)} = \bm \beta^{(i)} = 0$ and $\sum_{j} (c_j^{(k)} \coef_j^{(i)})^2 > 0$.  For any $i$ such that $f_{c^{(k)}}(\coef^{(i)}) = 0$ and $\mathbf \beta^{(i)}_k = 0$, similar to the proof of Theorem 3, $\beta^{(i)}$ lie in a manifold of dimension at most $\dim - 2$. Therefore, we could find a $w$ such that $w_k = 0$, $\|w\|_2 = 1$, and $\sum_j w_j \mathbf \beta^{(i)}_j = 0$ for any $\mathbf \beta$ satisfying $f_{c^{(k)}}(\coef^{(i)}) = 0$ and $\mathbf \beta^{(i)}_k = 0$. By Proposition \ref{prop:param}, this $w$ corresponds to a $\dict'$ and because $\dict$ is sharp local min with sharpness at least $\epsilon$, we know $\norm{\inv{\dict'}\dict}_{\beta} > \epsilon \|w\|_2 = \epsilon$. That shows $\frac{1}{n}\sum_i |\sum_j w_j \beta^{(i)}_j|f_{c^{(k)}}(\coef^{(i)}) \geq \frac{1}{n}\sum_i |\sum_j w_j \beta^{(i)}_j|\1(\beta_k^{(i)} = 0)  > \epsilon'$. Because $\norm{\dict^{-1}}_2 \leq \rho$ and $\|\coef^{(i)}\|_1$ is bounded, $|\sum_j w_j \beta^{(i)}_j|$ is bounded, too. That shows  $\max \{|\sum_j w_j \beta^{(i)}_j| \}\frac{1}{n}\sum_i f_{c^{(k)}}(\coef^{(i)} > \frac{1}{n}\sum_i |\sum_j w_j \beta^{(i)}_j|f_{c^{(k)}}(\coef^{(i)} > \epsilon'$, which shows $\frac{1}{n}\sum_i f_{c^{(k)}}(\coef^{(i)} > \epsilon'/\max \{|\sum_j w_j \beta^{(i)}_j| \}$. That completes the proof.
\end{proof}

\begin{proof}[Proof of Proposition \ref{prop:sharp_is_stable}]
$1) \leftrightarrow 2)$: by Theorem \ref{Thm:local}, we know 1) is equivalent to 
\[
\norm{B(\bm \beta, M)}_{\coef}^* < 1.
\]
Because of the definition of $\norm{\cdot}_{\coef}$, this condition is equivalent to for any $k=1,\ldots, \dim$,
and $\delta_{k,j}\in \R$ for $j = 1,\ldots, \dim$, $j \neq k$ such that $\sum_{j\neq k}\delta_{k,j}^2 > 0$, 
\[
\Big| \sum_{j} \delta_{k,j} \beta_j\Big|\1(\beta_k = 0) + \sum_{k,j}\delta_{j}\beta_{j}\sgn(\beta_k) + M_{k,j}\E |\beta_k| > 0.
\]
On the other hand, the left hand side is exactly the directional derivative of the optimization \ref{Eq:coordinatewise_sharp} at $\I_j$ along direction $(\delta_1,\ldots, \delta_\dim)$. Because every directional derivative is strictly positive, it is equivalent to 2).

$2)\leftrightarrow 3)$. We have already shown that 2) is equivalent to 
\[
\Big| \sum_{j} \delta_{k,j} \beta_j\Big|\1(\beta_k = 0) + \sum_{k,j}\delta_{j}\beta_{j}\sgn(\beta_k) + M_{k,j}\E |\beta_k| > 0.
\]
3) is equivalent to 
\[
\Big| \sum_{j} \delta_{k,j} \beta_j\Big|\1(\beta_k = 0) + \sum_{k,j}\delta_{j}\beta_{j}\sgn(\beta_k) + \tilde M_{k,j}\E |\beta_k| \geq 0.
\]
for any $|\tilde M_{k,h} - M_{k,h}|\leq \rho$. These two are clearly equivalent.
\end{proof}
\begin{proof}[Proof of Proposition \ref{prop:param}]
Without loss of generality, we only need to show $\|\inv{Q}_j\|_2 = 1$ for any $j = 1,\ldots, \dim$ when $k = 1$. We can write $Q = \Gamma \inv{\dict}$, where $\Gamma_{h, j} = w_h$ when $h = 1$ and $\Gamma_{h,j} = \sqrt{(w_h - M_{1,h})^2 + 1 - m_h^2}$ when $h\neq 1, j  = h$ and $0$ otherwise. Then $\inv{Q} =\dict \inv{\Gamma}$ and $$\inv{\Gamma}_{h,j} = \left\{
\begin{array}{cc} 1 & h = 1, j = 1\\
-w_j/(\sqrt{(w_h - M_{1,h})^2 + 1 - M_{1, h}^2}) & h = 1, j > 1\\
1/(\sqrt{(w_h - M_{1,h})^2 + 1 - M_{1, h}^2}) & h > 1, j = h\\
0 & h > 1, j \neq h
\end{array}\right.$$
For $h = 1$, $\|\inv{Q}_h\|_2^2 = \|\dict_1\|_2^2 = 1$. For any $h > 1$, $\|\inv{Q}_h\|_2^2 = \|w_h\dict_1 - \dict_h\|^2_2/((w_h - M_{1,h})^2 + 1 - M_{1, h}^2)=1$. That completes the proof.
\end{proof}
\begin{proof}[Proof of Proposition \ref{prop:monotone}]
Recall $f(\dict) = \sum_{i=1}^n \sum_{j = 1}^\dim \min(|\dict^{-1}[j,]\signal^{(i)}|, \tau)$. Denote $\beta^{(i)} = \inv{(\dict^{(t,j)})}\signal^{(i)}$ and define a new function
$\tilde{f}(\dict) =  \sum_{j = 1}^\dim \sum_{i=1, |\beta^{(i)}_j|\leq \tau }^n|\dict^{-1}[j,]\signal^{(i)}| + \sum_{i=1, |\beta^{(i)}_j|>\tau }^n\tau.$ Note that for any $\dict$, $\tilde f(\dict)$ is always no smaller than $f(\dict)$, that is, $\tilde f(\dict) \geq f(\dict)$. 
Because of Proposition \ref{prop:param}, we know the iterate $\dict^{(t, j+1)}$ in Algorithm \ref{Alg:bcd} is the optimal solution of the following optimization:
\begin{align*}
\textrm{argmin}_Q~ & \tilde{f}(Q^{-1})\\
\textrm{subject~to} ~& Q ~ \text{is parameterized as in Proposition \ref{prop:param}}.
\end{align*}
That means $\tilde f(\dict^{(t, j+1)}) \leq \tilde f(\dict^{(t, j)})$.
Note that $f(\dict^{(t, j)}) = \tilde{f}(\dict^{(t, j)})$ and $f(\dict^{(t, j+1)}) \leq \tilde{f}(\dict^{(t, j+1)})$. Thus, $f(\dict^{(t, j+1)}) \leq \tilde{f}(\dict^{(t, j+1)}) \leq \tilde{f}(\dict^{(t, j)}) = f(\dict^{(t, j)})$. That completes the proof.
\end{proof}
\section*{Appendix D: Parameter settings of dictionary learning algorithms in Section \ref{Sec:5.4}}
 \begin{itemize}
    \item EM-BiG-AMP: In EM-BiG-AMP, there is an outer loop that performs EM iterations and there is an inner loop. The outer loop is allowed up to 20 iterations. The inner loop is allowed a minimum of 30 and a maximum of 1500 iterations. 
    \item K-SVD: K-SVD has two parameters: number of iterations and the enforced sparsity. The number of iterations is set to be 1000. The enforced sparsity is set to be the same as the true sparsity of the underlying model $s$. 
    \item SPAMS: SPAMS optimizes an LASSO type objective iteratively. The number of iterations is set to be 1000 and the penalty parameter in front of the L1 norm is $\lambda = .1/\sqrt{N}$. 
    \item DL-BCD: Our algorithm has an outer loop and an inner loop. The outer loop is set to be at most 3. The inner loop is allowed a maximum of $100$ iterations.  $\tau$ is either $\infty$ or $0.5$. 
    \item ER-SpUD: We use the default settings in the package developed by the authors of ER-SpUD.
\end{itemize}
\end{document}